\documentclass[11pt]{article}
\usepackage[margin=1in]{geometry}
\usepackage{amsmath,amsthm,amsfonts,amssymb}
\usepackage{graphicx,enumerate}
\graphicspath{ {./figures/} }
\usepackage[title]{appendix}
\usepackage[
            CJKbookmarks=true,
            bookmarksnumbered=true,
            bookmarksopen=true,
            colorlinks=true,
            citecolor=blue,
            linkcolor=blue,
            anchorcolor=red,
            urlcolor=blue
            ]{hyperref}

\newcommand{\Bin}{\mathrm{Bin}}

\usepackage{tikz}
\usepackage{tikz-qtree}

\usetikzlibrary{calc}
\usetikzlibrary{decorations.pathreplacing,decorations.markings}

\tikzstyle{dot}=[circle,fill,black,inner sep=1pt]

\tikzset{
  on each segment/.style={
    decorate,
    decoration={
      show path construction,
      moveto code={},
      lineto code={
        \path [#1]
        (\tikzinputsegmentfirst) -- (\tikzinputsegmentlast);
      },
      curveto code={
        \path [#1] (\tikzinputsegmentfirst)
        .. controls
        (\tikzinputsegmentsupporta) and (\tikzinputsegmentsupportb)
        ..
        (\tikzinputsegmentlast);
      },
      closepath code={
        \path [#1]
        (\tikzinputsegmentfirst) -- (\tikzinputsegmentlast);
      },
    },
  },
  mid arrow/.style={postaction={decorate,decoration={
        markings,
        mark=at position .5 with {\arrow[#1]{stealth}}
      }}},
  early arrow/.style={postaction={decorate,decoration={
        markings,
        mark=at position .2 with {\arrow[#1]{stealth}}
      }}},
}

\newcommand{\singleedge}[2]{\path #1 edge[black] #2 ; \node[dot] at #1 {}; \node[dot] at #2 {};}

\def\alternatecolorred{%
    \pgfkeysalso{red}%
    \global\let\alternatecolor\alternatecolorblue 
}
\def\alternatecolorblue{%
    \pgfkeysalso{blue}%
    \global\let\alternatecolor\alternatecolorred 
}

\newcommand{\altred}{\let\alternatecolor\alternatecolorred 
\tikzset{every edge/.append code = {%
    \global\let\currenttarget\tikztotarget 
    \pgfkeysalso{append after command={(\currenttarget)}}
			\alternatecolor
}}
}
\newcommand{\altblue}{\let\alternatecolor\alternatecolorblue 
\tikzset{every edge/.append code = {%
    \global\let\currenttarget\tikztotarget 
    \pgfkeysalso{append after command={(\currenttarget)}}
			\alternatecolor
}}
}

\tikzstyle{vertexdot}=[circle, draw, fill=black, minimum size=3,inner sep=0pt]

\newtheorem{theorem}{Theorem}
\newtheorem{lemma}{Lemma}
\newtheorem{proposition}{Proposition}
\newtheorem{corollary}{Corollary}
\newtheorem{definition}{Definition}

\newtheorem{question}{Question}
\newtheorem{remark}{Remark} 

\newtheorem{claim}{Claim}

\usepackage{color}
\usepackage{subfigure}
\usepackage{algorithm}
\usepackage{algorithmic}

\usepackage{xspace,prettyref}

\newcommand{\reals}{{\mathbb{R}}}
\newcommand{\integers}{{\mathbb{Z}}}

\newcommand{\expect}[1]{\mathbb{E}\left[ #1 \right]}

\newcommand{\prob}[1]{ \mathbb{P}\left\{ #1 \right\} }

\newcommand{\Bern}{{\rm Bern}}
\newcommand{\Binom}{{\rm Binom}}
\newcommand{\Pois}{{\rm Pois}}

\newcommand{\eg}{e.g.\xspace}
\newcommand{\ie}{i.e.\xspace}

\newrefformat{eq}{(\ref{#1})}
\newrefformat{chap}{Chapter~\ref{#1}}
\newrefformat{sec}{Section~\ref{#1}}
\newrefformat{alg}{Algorithm~\ref{#1}}
\newrefformat{fig}{Fig.~\ref{#1}}
\newrefformat{tab}{Table~\ref{#1}}
\newrefformat{rmk}{Remark~\ref{#1}}
\newrefformat{clm}{Claim~\ref{#1}}
\newrefformat{def}{Definition~\ref{#1}}
\newrefformat{cor}{Corollary~\ref{#1}}
\newrefformat{lmm}{Lemma~\ref{#1}}
\newrefformat{prop}{Proposition~\ref{#1}}
\newrefformat{app}{Appendix~\ref{#1}}
\newrefformat{hyp}{Hypothesis~\ref{#1}}
\newrefformat{thm}{Theorem~\ref{#1}}
\newrefformat{ass}{Assumption~\ref{#1}}
\newrefformat{conj}{Conjecture~\ref{#1}}

\newcommand{\iprod}[2]{\left \langle #1, #2 \right\rangle}

\newcommand{\indc}[1]{{\mathbf{1}_{\left\{{#1}\right\}}}}

\newcommand{\calA}{{\mathcal{A}}}
\newcommand{\calB}{{\mathcal{B}}}
\newcommand{\calC}{{\mathcal{C}}}
\newcommand{\calD}{{\mathcal{D}}}
\newcommand{\calE}{{\mathcal{E}}}
\newcommand{\calF}{{\mathcal{F}}}
\newcommand{\calG}{{\mathcal{G}}}
\newcommand{\calH}{{\mathcal{H}}}
\newcommand{\calI}{{\mathcal{I}}}

\newcommand{\calN}{{\mathcal{N}}}

\newcommand{\calP}{{\mathcal{P}}}
\newcommand{\calQ}{{\mathcal{Q}}}
\newcommand{\calR}{{\mathcal{R}}}
\newcommand{\calS}{{\mathcal{S}}}
\newcommand{\calT}{{\mathcal{T}}}
\newcommand{\calU}{{\mathcal{U}}}

\newcommand{\calX}{{\mathcal{X}}}
\newcommand{\calY}{{\mathcal{Y}}}

\DeclareMathAlphabet{\varmathbb}{U}{bbold}{m}{n}

\renewcommand{\hat}{\widehat}
\renewcommand{\tilde}{\widetilde}

\usepackage{bbm}

\newcommand{\ER}{Erd\H{o}s-R\'{e}nyi\xspace}

\begin{document}

\pgfdeclarelayer{background}
\pgfdeclarelayer{foreground}
\pgfsetlayers{background,main,foreground}

\title{Seeded Graph Matching via Large Neighborhood Statistics}

\author{
Elchanan Mossel \\
MIT\\
{elmos@mit.edu}
\and 
Jiaming Xu\\
Duke University\\
{jiaming.xu868@duke.edu}
}

\maketitle

\begin{abstract}
We study a well known noisy model of the graph isomorphism problem. In this model, the goal is to perfectly recover the
vertex correspondence between two edge-correlated graphs, with an initial seed set of correctly matched vertex pairs revealed as side information. Specifically, the model first generates a parent graph $G_0$ from Erd\H{o}s-R\'{e}nyi random graph $\mathcal{G}(n,p)$  and then obtains two children graphs $G_1$ and $G_2$ by subsampling the edge set of $G_0$ twice independently with probability $s=\Theta(1)$. The vertex correspondence between $G_1$ and $G_2$ is obscured by randomly permuting the vertex labels of $G_1$ according to a latent permutation $\pi^*$. Finally, for each $i$, $\pi^*(i)$ is revealed independently with probability $\alpha$ as seeds.   

In the sparse graph regime where $np \le n^{\epsilon}$ for any $\epsilon<1/6$, we give a polynomial-time algorithm which perfectly recovers $\pi^*$, provided that $nps^2 - \log n \to +\infty$ and $\alpha \ge n^{-1+3\epsilon}$. 
This further leads to a sub-exponential-time, $\exp \left( n^{O(\epsilon) } \right)$, matching algorithm even without seeds. On the contrary, if $nps^2 - \log n =O(1),$ then perfect recovery is information-theoretically impossible as long as $\alpha$ is bounded away from $1$. 

In the dense graph regime, where $np = b n^{a}$, for fixed constants $a, b \in (0,1]$, we give a polynomial-time algorithm which succeeds when $b=O(s)$ and $\alpha = \Omega \left( (np)^{-\lfloor 1/a \rfloor } \log n\right) $. In particular, when $a=1/k$ for an integer $k \ge 1$, $\alpha = \Omega(\log n/n)$ suffices, yielding a quasi-polynomial-time $n^{O(\log n)}$ algorithm matching the best known algorithm by Barak et al. for the problem of graph matching without seeds when $k \geq 153$ and extending their result to new values of $p$ for $k=2,\ldots,152$.  

Unlike previous work on graph matching, which used small neighborhoods or small subgraphs with a logarithmic number of vertices in order to match vertices, our algorithms match vertices if their large neighborhoods have a significant overlap in the number of seeds.

\newpage

\end{abstract}

\section{Introduction}

In this paper, we study a well-known model of noisy graph isomorphism. Our main interest is in polynomial time algorithms for seeded problems where the matching between a small subset of the nodes is revealed. For seeded problems, our result provides a dramatic improvement over previously known results. Our results also 
shed light on the unseeded problem. In particular, we give (the first) sub-exponential time algorithms for sparse models and an $n^{O(\log n)}$ algorithm for dense models for some parameters, 
including some that are not covered by recent results of Barak et al.~\cite{barak2018nearly}.  

We recall that two graphs are isomorphic if there exists an edge-preserving bijection between their vertex sets.  The Graph Isomorphism problem is not known to be solvable in polynomial time, except in special cases such as graphs of bounded degree~\cite{luks1980isomorphism} and bounded eigenvalue multiplicity~\cite{babai1982isomorphism}.  However, a recent breakthrough of Babai~\cite{Babai2016} gave a quasi-polynomial time algorithm.

In a number of applications including network security~\cite{narayanan2009anonymizing,narayanan2008robust}, systems biology~\cite{singh2008global}, computer vision~\cite{conte2004thirty,schellewald2005probabilistic}, and natural language processing~\cite{haghighi2005robust},
we are given two graphs as input which we believe have an underlying isomorphism between them.  However, they are not exactly isomorphic because they have each been perturbed in some way, adding or deleting edges randomly. 
This suggests a noisy version of Graph Isomorphism also known as \emph{graph matching}~\cite{Livi2013}, where we seek a bijection that minimizes the number of edge disagreements.

Given two graphs with adjacency matrices $G_1$ and $G_2$, if our goal is to minimize the $\ell_2$ distance between $G_1$ and 
some permuted version of $G_2$, then graph matching can be viewed 
as a special case of the \emph{quadratic assignment problem} (QAP)~\cite{burkard1998quadratic}: namely, 
\begin{align}
\min_{\Pi}  \| G_1 - \Pi G_2 \Pi^\top \|^2_F , \label{eq:QAP}
\end{align}
where $\Pi$ ranges over all $n \times n$ permutation matrices, and $\| A \|^2_F = \sum_{ij} A_{ij}^2$ denotes the Frobenius norm.  QAP is NP-hard in the worst case.  There are 
exact search methods for QAP based on branch-and-bound and cutting planes, as well as various approximation algorithms  
based on linearization schemes, and convex/semidefinite programming relaxations (see~\cite{feizi2016spectral} and the references therein).  However, approximating QAP within a factor $2^{\log^{1-\epsilon} (n) }$ for $\epsilon > 0$ is NP-hard~\cite{makarychev2010maximum}. 

These hardness results only apply in the worst case, where the two graphs are designed by an adversary.  However, in many 
aforementioned applications, we are not interested in worst-case instances, but rather in instances for which
there is enough information in the data to recover the underlying isomorphism, \ie, when the amount of data or signal-to-noise ratio is above the information-theoretic limit. The key question is whether there exists an efficient algorithm that is successful all the way down to this limit. 
In this vein, we consider the following random graph model denoted by $\calG(n,p;s)$~\cite{Pedarsani2011}.

\begin{definition}[The Correlated \ER model $\calG(n,p;s)$]
Suppose we generate a parent graph $G_0$ from the \ER random graph model $\calG(n,p)$. For a fixed 
realization of $G_0$, we generate two subgraphs $G_1$ and $G_2$ by subsampling the edges of $G_0$ 
twice. More specifically, 
\begin{itemize}
\item We let $G^*_1$ be a random subgraph of $G_0$ obtained by including every
edge of $G_0$ with probability $s$ independently. 
\item We repeat the above subsampling procedure, but independently
to obtain another random subgraph of $G_0$, denoted by $G_2$. 
\end{itemize}
To further model the scenario that we do not know the vertex correspondence between $G_1$ and
$G_2$ a prior, we sample a random permutation $\pi^*$ over $[n]$ and let $G_1$ denote the
graph obtained by relabeling every vertex $i$ in $G^*_1$ as $\pi^*(i).$  
\end{definition}
The goal is to exactly recover $\pi^*$ from the observation of $G_1$ and $G_2$ with high probability, \ie,
to design an estimator $\hat{\pi}$ based on $G_1$ and $G_2$ such that
$$
\prob{ \hat{\pi} (G_1, G_2) = \pi^* } \to 1, \quad \text{ as } n \to \infty.
$$

As a motivating example, we can model $G_0$ as some true underlying friendship network of $n$ persons, 
$G_1$ is an anonymized Facebook network of the same set of persons, and $G_2$ is a Twitter network with
known person identities. If we can recover the vertex correspondence between $G_1$ and $G_2$, then we
can de-anonymize the Facebook network $G_1$ (this example ignores many important facts such as additional graph structures in real life networks).

Note that $s$ is equal to the probability of $e\in E(G_2)$ conditional on $e \in E(G_1)$, and hence can be viewed as a measure of the edge correlations. Throughout this paper, without further specifications, we shall assume $s=\Theta(1)$.

In the fully sampling case $s=1$, graph matching under $\calG(n,p;1)$ reduces to 
the Graph Automorphism problem for \ER\ graphs. In this case, a celebrated result~\cite{wright1971graphs} shows that
if $ \log n+ \omega(1)  \le np \le n- \log n -\omega(1) $,
then with probability $1-o(1)$, 
the size of the automorphism group  of $G_0$ is $1$ and hence the underlying permutation $\pi^*$
can be exactly recovered; otherwise, with probability $1-o(1)$, the size of the automorphism group  of $G_0$ is 
strictly bigger than $1$ and hence exact recovery of the underlying permutation is information-theoretically
impossible. 
Recent work~\cite{cullina2016improved,cullina2017exact}\footnote{
In fact, a more general correlated \ER random graph model is considered in~\cite{cullina2016improved,cullina2017exact}, where $\prob{G_1(i,j)=a, G_2(i,j)=b} =p_{a,b}$ for $a,b\in\{0,1\}$. } has extended this result to the partially sampling case $s=\Theta(1)$ and
$p \le 1/2$, showing that the Maximum Likelihood Estimator, or equivalently the optimum of QAP \prettyref{eq:QAP}, coincides with the 
ground truth $\pi^*$ with high probability, provided that $nps^2 \ge \log n + \omega(1)$; on the contrary, any estimator is correct with probability $o(1)$, if $nps^2 \le \log n - \omega(1)$.

From a computational perspective, in the fully sampling case $s=1$, there exist linear-time algorithms 
which attain the recovery threshold, in the sense that they exactly recover the underlying permutation with high probability whenever $np= \log n+\omega(1)$~\cite{bollobas1982distinguishing,czajka2008improved}. However, 
in the partially sampling case, it is still open whether any efficient algorithm can succeed close to the threshold. 
A recent breakthrough result~\cite{barak2018nearly} obtains a quasi-polynomial-time ($n^{O(\log n)}$) algorithm which
succeeds when $np \ge n^{o(1)}$ and $s \ge (\log n)^{-o(1)}.$ However, this is still far away from the information-theoretic limit $nps^2 \ge \log n + \omega(1)$.

Another line of work~\cite{Pedarsani2011,yartseva2013performance,korula2014efficient,Lyzinski2013Seeded,Fishkind2018Seeded,Shirani2017Seeded} 
in this area considers a relaxed version of the graph matching problem, 
where an initial seed set of correctly matched vertex pairs is revealed as side information. This is motivated by the fact that in many
real applications, some side information on the vertex identities are available and have been successfully utilized to match many real-world networks~\cite{narayanan2009anonymizing,narayanan2008robust}. Formally, in this paper, 
we assume the seed set
is randomly generated as follows. 
\begin{definition}[Seeded graph matching under $\calG(n,p;s,\alpha)$]
In addition to $G_1, G_2$ that are generated under $\calG(n,p;s)$ with a latent permutation $\pi^*$, 
we have access to $\pi_0$ such that $\pi_0(i)=\pi^*(i)$ with probability $\alpha$
and $\pi_0(i)=?$ with probability $1-\alpha$ independently across different $i.$
The goal is to recover $\pi^*$ based on $G_1,$ $G_2,$ and $\pi_0.$
\end{definition}
The vertex $i$ such that $\pi_0(i)=\pi^*(i)$ is called seeded vertices and 
the set of seed vertices is denoted by $I_0$. Note that according to our model, 
the number of seeds $|I_0|$ 
is distributed as $\Binom(n,\alpha)$.  For a given size $K$, we could also consider a deterministic size model where $I_0$ is chosen uniformly at random from all possible subsets of 
$[n]$ with size $K$. The main results of this paper readily extend to this deterministic size model with 
$K=\lfloor n \alpha \rfloor$. 

The results of the seeded graph matching turn out to be useful for 
designing graph matching without seeds. On the one hand, when a seed set of size $K$ is not given, we could obtain it in $n^{O(K)}$ steps by randomly choosing a set of $K$ vertices and then enumerating all the possible mapping. This is known as the beacon set approach to graph isomorphism~\cite{Lipton1978}. On the other hand, we could first apply a seedless graph matching algorithm and then apply a seeded graph matching algorithm to boost its accuracy. This two-step algorithms have been successful both theoretically~\cite{babai1980random}~\cite[Section 3.5]{bollobas1998random} and empirically~\cite{Lyzinski2013Seeded}.

In the sparse graph regime $np=\Theta(\log n)$, it is shown in~\cite{yartseva2013performance} 
that if $\alpha =\Omega(1/\log^2 n)$, or equivalently, the size of the seed set is $\Omega(n/\log^2 n)$, 
then a percolation-based graph matching
algorithm correctly matches $n-o(n)$ vertices  in polynomial-time with high probability. In the dense graph regime $np=n^{\delta}$ for some
constant $\delta \in (0,1)$, a seed set of size $
\Theta(n^{1-\delta})$ suffices as shown in~\cite{yartseva2013performance}.  
Another work~\cite{korula2014efficient}
shows that if $nps^2 \alpha \ge 24 \log n $, then one can match all vertices correctly in polynomial-time with high probability
based on counting the number of ``common'' seeded vertices. 
Note that this exact recovery result requires the seed set size to be
linear in $n$ in the sparse graph regime $np=\Theta(\log n)$. 

In summary, despite a significant amount of previous work on seedless and seeded graph matching, the following two fundamental questions remain elusive:

\begin{question}
In terms of graph sparsity, can we achieve the information-theoretic limit $nps^2 - \log n \to +\infty$ in sub-exponential, or polynomial time?
\end{question}
\begin{question}
In terms of seed set, what is the minimum number of seeds required for exact recovery in sub-exponential, or polynomial time?
\end{question}

Our main results shed light on this two questions by improving the state-of-the-art 
of seeded graph matching. 
First, we show that it is possible to achieve the information theoretic limit $nps^2 \ge \log n+\omega(1)$ of graph sparsity in polynomial-time.
Then, we show the number of seeds needed for 
exact recovery in polynomial-time can be as low as $n^{\epsilon}$ in the sparse graph regime ($np \le n^{\epsilon}$) and $\Omega(\log n)$ in the dense graph regime.

\subsection{Main Results}

We first consider the sparse graph regime. 

\begin{theorem}\label{thm:main}
Suppose $np \le n^{1/2-\epsilon}$ for a 
fixed constant $\epsilon>0$ and $s=\Theta(1)$. Assume 
\begin{align}
nps^2 - \log n & \to + \infty \label{eq:it_limit} \\
\alpha & \ge n^{-1/2+3\epsilon}. \label{eq:seed_limit}
\end{align} 
Then there exists a polynomial-time algorithm, namely \prettyref{alg:graph_matching_ell_hop}, which outputs $\hat{\pi}=\pi^*$ with probability at least $1-o(1)$ under the seeded $\calG(n,p;s,\alpha)$ model.
\end{theorem}

Notice that $\prettyref{eq:it_limit}$ is the information-theoretic limit
for graph matching under the seedless $\calG(n,p;s)$ model. In fact, \prettyref{thm:converse} shows that 
\prettyref{eq:it_limit} is necessary for seeded graph matching
as long as $\alpha$ is bounded away from $1.$
Its proof is standard and can be found in 
\prettyref{app:converse}.

\begin{theorem}\label{thm:converse}
If 
$$
nps^2-\log n = O(1),  
$$
then 
any algorithm outputs $\hat{\pi} \neq \pi^*$ 
with at least a probability of $\Omega
\left((1-\alpha)^2 \right)$ 
under the seeded $\calG(n,p;s,\alpha)$ model.
\end{theorem}

Also, the condition~\prettyref{eq:seed_limit}
requires that the size of the seed set is $n^{1/2+3\epsilon}$ compared to the best previously known results that 
required the seed set to be almost linear in $n$. 

It is natural to ask if $n^{1/2}$ seeded nodes are required for polynomial time algorithm. 
While from the proof of~\prettyref{thm:main}, it might look that $n^{1/2}$ is optimal due to the birthday paradox effect, it turns out we can do better! 

The following result relaxes the size of seed set needed to 
$n^{3\epsilon}$.
\begin{theorem}\label{thm:sparse_improved}
Suppose $np \le n^{\epsilon}$ for a fixed constant $\epsilon< 1/6$ and $s=\Theta(1)$. Assume 
\begin{align}
nps^2 - \log n & \to + \infty \label{eq:it_limit} \\
\alpha & \ge n^{-1+3\epsilon}. \label{eq:seed_limit}
\end{align} 
Then there exists a polynomial-time algorithm, namely \prettyref{alg:graph_matching_ell_hop_witness_sparse}, which outputs $\hat{\pi}=\pi^*$ with probability at least $1-o(1)$ under the seeded $\calG(n,p;s,\alpha)$ model.
\end{theorem}

We next consider the dense graph regime, where we assume the average degree $np$ is parameterized as: 
\begin{align}
np = b n^{a} \label{eq:dense_regime}
\end{align} 
for some fixed constants $a, b \in (0,1]$. 
Let 
\begin{align}
d = \left \lfloor \frac{1}{a} \right\rfloor +1, \label{eq:def_d}
\end{align}

\begin{theorem}\label{thm:dense}
Consider the dense graph regime \prettyref{eq:dense_regime}.
Assume
\begin{align}
b \le \frac{s}{16(2-s)^2},  \label{eq:assumption_b}
\end{align}
and 
\begin{align}
\alpha \ge  \frac{ 300 \log n}{(nps^2)^{d-1}}, \label{eq:assumption_alpha_dense}
\end{align}
where $d$ is given in \prettyref{eq:def_d}. 
Then there exists an polynomial-time algorithm, namely \prettyref{alg:graph_matching_ell_hop_witness}, which 
outputs $\hat{\pi} =\pi^*$ with probability $1-4n^{-1}$
under the seeded $\calG(n,p;s,\alpha)$ model.
\end{theorem}

Our results for seeded graph matching also imply the results for graph matching without seeds. 
\begin{theorem}\label{thm:seedless_graph_matching}
Suppose a Seeded Graph Matching algorithm 
outputs $\hat{\pi}=\pi^*$ with high probability under the seeded graph matching model $\calG(n,p;s,\alpha)$. Assume 
$
nps^2 -\log n \to +\infty
$
and $\alpha n \to +\infty$. 
Then there exists an algorithm, namely \prettyref{alg:graph_matching_seedless}, which calls the Seeded Graph Matching algorithm $n^{O(\alpha n) }$ times and outputs 
$\hat{\pi}=\pi^*$ under the seedless model $\calG(n,p;s)$ with high probability. 
\end{theorem}

\begin{remark}
Consider the dense regime 
\prettyref{eq:dense_regime} with $a =1/k$ for an integer $k \ge 1$. 
Then $d=k+1$ and 
$(np)^{d-1} = b^{k} n$. Hence, as shown by \prettyref{thm:dense}, 
$\alpha n \ge 300 \log n (bs^2)^{-k}$,
or equivalently $\Omega(\log n)$ number of 
seeds, suffice for exact recovery in polynomial-time.  Since we can enumerate over
all possible matchings for $\log n$ seeds in 
quasi-polynomial $ n^{ O( \log n) }$ time, this implies a quasi-polynomial 
time matching algorithm even without seeds, as shown by \prettyref{thm:seedless_graph_matching}. 
The previous work~\cite{barak2018nearly} gives a quasi-polynomial time
matching algorithm in the range 
$$ 
np \in \left[ n^{o(1)}, n^{1/153}] \cup [n^{2/3}, n^{1-\epsilon} \right].
$$
Our results complement their results by filling in gaps in the above
range with points $np \in \{ b n^{1/k}: 1 \le k \le 152 \}$. 
\end{remark}

\subsection{Key Algorithmic Ideas and Analysis Techniques}
Most previous work~\cite{Pedarsani2011,yartseva2013performance,korula2014efficient,Lyzinski2013Seeded,Fishkind2018Seeded,Shirani2017Seeded} on seeded graph matching exploits the seeded information by looking at the number of seeded vertices that are \emph{direct} neighbors of a given vertex. Since the average degree of a vertex is $np$, $np \alpha \gg 1$ is needed so that there are sufficiently many seeded vertices that are direct neighbors of a given vertex.

Our idea is to explore much bigger (``global'') neighborhoods of a given vertex up to radius $\ell$ for a suitably chosen $\ell$, and match two vertices 
by comparing the set of seeded vertices 
in their $\ell$-th local neighborhoods. 
This idea was used before in the noiseless and seedless case, in~\cite{bollobas1982distinguishing,czajka2008improved} but to the best of our knowledge was not used
in the noisy and seeded case. 
Since we are looking at global neighborhoods, we can only perform very simple tests. Indeed, the test we perform to check if two vertices are matched is just to count how many seeded vertices do the two neighborhoods have in common. Thus, our algorithms are very simple. 

The main challenge in the analysis is to control the size of neighborhoods of the coupled graphs $G_0, G_1$ and $G_2$. In this regard, we draw on a number of tools from the literature on studying subgraph counts~\cite{Janson1997Random} and the diameter in random graphs~\cite{bollobas1998random}.
See \prettyref{app:neighborhood_expansion} for details. 


\section{Our Algorithms}




Before presenting our algorithms, we first explain why \prettyref{eq:it_limit} is needed for graph matching
under $\calG(n,p,s)$. Denote the intersection graph 
and the union graph by $G_1^* \wedge G_2$ and $G_1^* \vee G_2.$
Then 
$$
G_1^* \wedge G_2 \sim \calG(n, ps^2) \quad \text{and} \quad  
G_1^* \vee G_2 \sim \calG(n, ps(2-s)).
$$
Notice that  $G_1^* \wedge G_2$ contains the statistical signature
for matching vertices, as a subgraph in $G_1^* \wedge G_2$ will appear
in both $G_1$ and $G_2.$ If $nps^2 -\log n =O(1)$, then classical random
graph theory implies that with high probability, 
$G_1^* \wedge G_2$ contains isolated vertices. The underlying true vertex correspondence of these isolated vertices cannot be correctly matched; hence the impossibility of exact recovery. See \prettyref{app:converse} for a precise argument.  

In contrast, if $nps^2 -\log n \to +\infty$, then $G_1^* \wedge G_2$ is connected
with high probability. Moreover, for a high-degree vertex $i$ in $G_1^* \wedge G_2$, its local neighbhorhood  grows like a branching process.
In particular, the number of
vertices at distance $\ell$ from $i$ is approximately $(nps^2)^\ell$.
Furthermore, for a pair of two vertices $i,j$ 
chosen at random in $G_1^* \vee G_2$, 
the intersection of the local neighborhoods of $i$ and $j$ is
typically of size $O \left( (nps)^{2\ell} n^{-1} \right)$. Therefore, if $(nps^2)^\ell \gg (nps)^{2\ell} n^{-1}$ and $\alpha (nps^2)^\ell \gg 1$, a large number of vertices can be distinguished with high probability based on the set of seeded vertices in their $\ell$-th local neighborhoods. This is the key idea underlying our algorithms.

We shall use the following notations of local neighborhoods. For a given graph $G$, we denote by $\Gamma^G_k(u)$ the set of vertices at distance
$k$ from $v$ in $G$:
\begin{align}
\Gamma^G_k (u) = \{ v \in V(G): d(u,v) = k \}
\label{eq:def_Gamma_neighbor}
\end{align}
and write $N^G_k(u)$ for the set of vertices within distance $k$ from $u$:
\begin{align}
N^G_k (u) = \cup_{i=0}^k \Gamma_i (u).
\label{eq:def_N_neighbor}
\end{align}
When the context is clear, we abbreviate 
$\Gamma^G_k (u)$ and $N^G_k (u)$ as $\Gamma_k(u)$ 
and $N_k(u)$ for simplicity.

\subsection{A Simple Algorithm in Sparse Graph Regime}
We first present a simple seeded graph matching algorithm which succeeds up to the information-theoretic limit in terms of graph sparsity when the initial seed set is of size $n^{1/2+3\epsilon}$. 

The idea of the algorithm is based on matching $\ell$-th local neighborhoods of two vertices by finding independent paths (vertex-disjoint except for the starting vertex) to seeded vertices. The $\ell$ is chosen such that $(np)^\ell \approx n^{1/2-\epsilon}$. In this setting, we 
expect that if $i$ in $G_1$ and $j$ in $G_2$ are true matches, then their local neighborhoods intersect a lot; 
if $i$ and $j$ are wrong matches, then their local
neighborhoods barely intersect. Hence, if $\alpha (nps^2)^\ell \gg 1$,
then we can find a sufficiently large number of, say $m$, 
independent (vertex-disjoint except for $i$) paths of length 
$\ell$ from $i$ to $m$ seeded vertices in $\Gamma^{G_1^* \wedge G_2}_\ell (i)$. 
Such $m$ paths of length $\ell$
form a starlike tree $T$ in $G_1^* \wedge G_2$ with root vertex $i$ and 
a set of $m$ seeded leaves, denoted by $L$ (See 
 \prettyref{fig:STexample} for an example of $m=3$ and $\ell=2$). 
Note that $T$ will appear in $G_2$ with root vertex $i$
and the set of seeded leaves $L$; it will also appear in 
$G_1$ with root vertex $\pi^*(i)$ and 
the corresponding set of seeded leaves $\pi^*(L)$.  
However, since the $\ell$-th local neighborhoods of two distinct vertices barely intersect,  
$T$ will \emph{not} appear in $G_1^* \vee G_2$ with a
root vertex other than $i$. Therefore, we can correctly match the vertex $\pi^*(i)$ in $G_1$ with the 
high-degree vertex $i$ in $G_2$
by finding such a starlike tree $T$, or equivalently $m$ independent $\ell$-paths  to
a set of $m$ common seeded vertices.

\begin{algorithm}
\caption{Graph matching based on counting independent $\ell$-paths to seeded vertices
}\label{alg:graph_matching_ell_hop}
\begin{algorithmic}[1]
\STATE {\bfseries Input:} $G_1$, $G_2$, $\pi_0$, $m, \ell \in \integers$ 
\STATE {\bfseries Output:}  $\hat{\pi}$.
\STATE {\bfseries Match high-degree vertices:} 
For each pair of unseeded vertices $i_1 \in V(G_1)$ 
and $i_2 \in V(G_2)$, 
if there are $m$ independent $\ell$-paths in $G_2$ from $i_2$ to a set of $m$ seeded vertices $L \subset \Gamma^{G_2}_\ell(i_2)$, 
and there are 
$m$ independent $\ell$-paths in $G_1$ from $i_1$ to the corresponding set of $m$ seeded vertices $\pi_0(L) \subset
\Gamma^{G_1}_\ell(i_1)$, 
then 
set $\hat{\pi}(i_2)=i_1$. 
Declare failure if there is any conflict. 
\STATE {\bfseries Match low-degree vertices:} For every $i_2 \in I_0$, set $\hat{\pi}(i_2)=\pi_0(i_2)$. 
For all the pairs of unmatched vertices $(i_1,i_2)$, if $i_1$ is adjacent to a matched vertex $j_1$ in $G_1$ 
and $i_2$ is adjacent to vertex $\hat{\pi}(j_1)$ in $G_2$, set $\hat{\pi}(i_2)=i_1$. 
Declare failure if there is any conflict. 
\STATE Output $\hat{\pi}$ to be a random permutation when failure is declared or there is any vertex unmatched. 
\end{algorithmic}
\end{algorithm}

There are two tuning parameters $\ell$ and $m$ in
\prettyref{alg:graph_matching_ell_hop}. Later in our analysis, we will optimally choose
\begin{align}
\ell = \left\lfloor  \left( \frac{1}{2} - \epsilon \right) \frac{\log n}{\log (nps^2) } \right\rfloor \ge 1 \label{eq:def_ell_1}
\end{align}
and 
\begin{align}
m = \left \lceil \frac{2}{\epsilon} \right\rceil \label{eq:def_m}.
\end{align}

Note that when $nps^2 -\log n \to +\infty$, there may exist vertices with small degrees. 
In fact, classical random graph results say that the minimum degree of 
$\calG(n,p)$ is $k$ with
high probability for a fixed integer $k$, provided that 
$$ 
(k-1) \log \log n +\omega(1) \le 
nps^2 - \log n  \le  k \log \log n -\omega(1) ,
$$
see, \eg, \cite[Section 4.2]{frieze2015}.
Hence, due to the existence of low-degree vertices, we may not be able to 
match all vertices correctly at one time based on the number of independent
paths to seeded vertices. Our idea is to first match high-degree vertices and then match the remaining low-degree vertices
with the aid of high-degree vertices matched in the first step. In particular, we let 
\begin{align}
\tau = \frac{nps^2}{\log (nps^2) }. \label{eq:def_tau}
\end{align}
We say a vertex $i$ high-degree, if its degree $d_i \ge \tau $ in $G_1^*\wedge G_2$; 
otherwise, we say it is a low-degree vertex. As we will see in \prettyref{sec:sparse_simple}, conditioning  on  that 
$G_1^* \wedge G_2$ and $G_1^* \vee G_2$ satisfy some 
typical graph properties, all low-degree vertices can be
 easily matched correctly given a correct matching of high-degree vertices. 


In passing, we remark on the time complexity of \prettyref{alg:graph_matching_ell_hop}.
Note that for ease of presentation, in \prettyref{alg:graph_matching_ell_hop}, we do not specify  how to efficiently find out  
whether there exist $m$ independent $\ell$-paths in $G_2$ from $i_2$ to seed set $L \subset \Gamma^{G_2}_\ell(i_2)$, and $m$ independent 
$\ell$-paths in $G_1$ from $i_1$ to the corresponding seed set $\pi_0(L) \subset \Gamma^{G_1}_\ell(i_1)$. 
It turns out for a given pair of vertices $i_1, i_2$, this task can be reduced to a maximum flow problem in a directed graph, which can be solved via Ford--Fulkerson algorithm~\cite{ford1956maximal} 
in $O(n\alpha)$ time steps (see \prettyref{app:time_complexity_1} for details). Since there are at most $n^2$ pairs of vertices $i_1, i_2$ to consider, Step 3 of \prettyref{alg:graph_matching_ell_hop}
taks at most $O( n^3 \alpha )$. 
The Step 4 of matching low-degree vertices in \prettyref{alg:graph_matching_ell_hop} takes at most $O(n^3 p)$ time steps. Hence, in total  \prettyref{alg:graph_matching_ell_hop} takes at
most $O\left(n^3 (\alpha+p) \right)$ time steps. 

\subsection{A Simple Algorithm in Dense Graph Regime}

In this subsection, we consider the dense graph regime given in \prettyref{eq:dense_regime},
where $np=bn^{a}$ and $d=\lfloor 1/a\rfloor +1$. In this setting, since 
$p^d n^{d-1} - 2 \log n\to +\infty $ and $p^{d-1}n^{d-2}-2\log n \le -\infty$, it follows from \cite[Corollary 10.12]{bollobas1998random} that $\calG(n,p)$ has diameter $d$ with high probability. Thus, when $s=\Theta(1)$, both $G_1^* \wedge G_2$
and $G_1^* \vee G_2$  have diameter $d$ with high probability. 
Therefore, we present an algorithm based on 
matching the $d-1$-th local neighborhood of 
two vertices. 
%
More specifically, our algorithm matches $i_1 \in V(G_1)$ and vertex $i_2 \in V(G_2)$ based on the number of seeded vertices \emph{within} distance $d-1$ from $i_1$ in $G_1$ and \emph{within} distance $d-1$ from $i_2$ in $G_2$.

\begin{algorithm}
\caption{Graph matching based on $(d-1)$-hop witness in dense regime}\label{alg:graph_matching_ell_hop_witness}
\begin{algorithmic}[1]
\STATE {\bfseries Input:} $G_1$, $G_2$, $\pi_0$, $d \in \integers$. 
\STATE {\bfseries Output:}  $\hat{\pi}$.
\STATE {\bfseries Match all vertices:} For each pair of unseeded vertices $i_1 \in V(G_1)$ and $i_2 \in V(G_2)$, compute
\begin{align}
w_{i_1,i_2} = \left| \left\{ j \in I_0: \pi_0(j) \in N^{G_1}_{d-1} (i_1), \;
j \in N^{G_2}_{d-1} (i_2) \right\} \right|. 
\label{eq:def_witness_w}
\end{align}
Set $\hat{\pi}(i_2) \in \arg \max_{i_1}w_{i_1,i_2}$. Set $\hat{\pi}(i_2)=\pi_0(i_2)$ for each seeded vertex $i_2 \in I_0$. 
Declare failure if there is any conflict.
\end{algorithmic}
\end{algorithm}

\prettyref{alg:graph_matching_ell_hop_witness} runs in polynomial-time. The precise running time depends on the data structures for storing and processing graphs. To be specific, let us assume it takes one time step to fetch the set of direct neighbors of a given vertex. Then
fetching the set $N^{G}_\ell(i)$ of all vertices within distance $\ell$ from a
given vertex $i$ takes 
a total of $O(|N^{G}_\ell(i)|)=O(n)$ time steps. Thus computing $w_{i_1,i_2}$ in
\prettyref{eq:def_witness_w} for a given
pair of vertices $i_1, i_2$ takes at most $O(n)$ time steps. Hence, in total 
\prettyref{alg:graph_matching_ell_hop_witness} takes $O(n^3)$ time steps. One could possibly obtain a better running time via a more careful analysis or a better data structure.

The difference in the analysis compared to the first algorithm is that the $(d-1)$-th local neighborhoods are not tree-like anymore. Instead, we have to analyze the exposure process of the two neighborhoods, for which we use a previous result of \cite[Lemma 10.9]{bollobas1998random}
in studying the diameter of random graphs. 

\subsection{An Improved Algorithm in Sparse Graph Regime}
In the sparse regime where $np$ is poly-logarithmic, \prettyref{alg:graph_matching_ell_hop_witness} does not perform well. 
This is because for two distinct vertices $u,v$ that are close by, 
their $\ell$-th local neighborhoods have a large overlap, \ie, 
$| N^{G}_\ell(u) \cap N^{G}_\ell(v)|$ is not much smaller than 
$| N^G_\ell(u) |$ or $| N^{G}_\ell(v)|$, rendering  $w_{i_1,i_2}$ given in \prettyref{eq:def_witness_w} ineffective to distinguish $u$ from $v$. 

However, in the sparse regime, distinct vertices $u,v$ only have very few common neighbors. Moreover, if we remove vertices $u, v$, 
the non-common neighbors become far apart, 
and for distinct vertices far apart, their local neighborhoods only have a small overlap.   
Therefore, we expect most of $u,v$'s neighbor's $\ell$-th local neighborhoods (after removing vertices $u,v$) 
do not have large
intersections for a suitably chosen $\ell$. This gives rise to~\prettyref{alg:graph_matching_ell_hop_witness_sparse}.

\begin{algorithm}
\caption{Graph matching based on $(d-1)$-hop witness in sparse regime}\label{alg:graph_matching_ell_hop_witness_sparse}
\begin{algorithmic}[1]
\STATE {\bfseries Input:} $G_1$, $G_2$, $\pi_0$, $\ell \in \integers$, $\eta \in \reals_+$.
\STATE {\bfseries Output:}  $\hat{\pi}$.
\STATE {\bfseries Match high-degree vertices:} 
For all the pairs of unseeded vertices $(u,v)$ and for each pair of their neighbors $(i,j)$ with $i \in \Gamma_1^{G_1}(u)$ and $j \in \Gamma_1^{G_2} (v)$, compute
\begin{align}
w^{u,v}_{i,j} = 
\min_{x \in V(G_1), y \in V(G_2) }
\left\{ 
\left| \left\{ k \in I_0: \pi_0(k) \in N^{G_{1} \backslash \{ u, x\} }_{\ell} (i), \;
k \in N^{G_{2} \backslash \{v,y\} }_{\ell} (j) \right\} \right| \right\} , \label{eq:def_w_sparse}
\end{align}
where  $G\backslash S $ denotes $G$ with 
set of vertices $S$ removed. Let 
\begin{align}
Z_{u,v} = \sum_{i \in \Gamma_1^{G_1}(u)} \sum_{j \in \Gamma_1^{G_2} (v) } \indc{ w^{u,v}_{i,j} \ge \eta}.
\label{eq:Z_sparse}
\end{align}
If $ Z_{u,v} \ge \log n/\log \log n-1$, set $\hat{\pi}(v) = u$.
Declare failure if there is any conflict. 
\STATE The remaining two steps are the same as \prettyref{alg:graph_matching_ell_hop}.
\end{algorithmic}
\end{algorithm}

Note that in computing the number of seeded vertices within distance $\ell$ from
both vertex $i$ in $G_1$ and vertex $j$ in $G_2$ in \prettyref{eq:def_w_sparse},
we remove vertices $u, x$ in $G_1$ and
vertices $v,y$ in $G_2$, and take the minimum over all possible choices of $x$ and $y$. As a result,
\begin{align}
w^{u,v}_{i,j} 
\le \left| \left\{ k \in I_0: \pi_0(k) \in N^{G_{1} \backslash \{ u, v\} }_{\ell} (i), \;
k \in N^{G_{2} \backslash \{u,v\} }_{\ell} (j) \right\} \right|, 
\label{eq:def_w_remove}
\end{align}
where the right hand side becomes independent from the 
edges incident to $u$ and $v$ in $G^*_1 \vee G_2$. This independence is crucial in our analysis to ensure that 
$Z_{u,v}$ is small for $u \neq \pi^*(v)$ via concentration inequalities of multivariate polynomials~\cite{Vu2002}. 

There are two tuning parameters $\ell$ and $\eta$ in 
\prettyref{alg:graph_matching_ell_hop_witness_sparse}. 
In our analysis later, we will optimally choose 
\begin{align}
\ell = \left \lfloor \frac{(1-  \epsilon) \log n}{\log (np ) } \right \rfloor, \label{eq:def_ell}
\end{align}
and 
\begin{align}
\eta =  4^{2\ell+2} n^{1-2 \epsilon} \alpha.  \label{eq:def_eta}
\end{align}

As for time complexity, \prettyref{alg:graph_matching_ell_hop_witness_sparse}
takes at most $O(n^{5+2\epsilon})$ time steps. To see this, similar to \prettyref{alg:graph_matching_ell_hop_witness}, if we assume one unit time to fetch a set of direct neighbors of a given vertex, then it takes at most $O(n^3)$ time steps to 
compute \prettyref{eq:def_w_remove} 
for given pairs of vertices $(u,v)$ and $(i,j)$. There are
at most $n^{2+2\epsilon}$ such pairs. 
The step of matching low-degree
vertices as specified in \prettyref{alg:graph_matching_ell_hop}
takes $O(n^3p)$ time steps in total. 
Thus in total \prettyref{alg:graph_matching_ell_hop_witness_sparse} takes at most $O(n^{5+2\epsilon} )$ time steps.

\subsection{Graph Matching without Seeds}
Even without an initial seed set revealed as side information, we can select a random subset of vertices $I_0$ in $G_1$ and enumerate all the possible mappings
$f: I_0 \to [n]$ from $I_0$ to vertices in $G_2$ in at most $n^{|I_0|}$ steps. Each of the possible mappings can be viewed as seeds; thus we can apply our seeded graph matching algorithm. Among all possible $n^{|I_0|}$ mappings, we finally output the best matching which minimizes the edge disagreements. See \prettyref{alg:graph_matching_seedless} for details.

\begin{algorithm}
\caption{Seedless Graph matching via Seeded Graph Matching}\label{alg:graph_matching_seedless}
\begin{algorithmic}[1]
\STATE {\bfseries Input:} $G_1$, $G_2$
\STATE {\bfseries Output:}  $\hat{\pi}$.
\STATE Select a random subset $I_0$ of $V(G_1)$ by including each vertex with probability $\alpha$. 
\STATE For every possible mapping $f: I_0 \to [n]$, run Seeded Graph Matching 
Algorithm with a seed set $I_0$, and output $\pi_f$. 
\STATE Output 
$$
\hat{\pi} \in \arg \min_{\pi_f} \| G_1 - \Pi_f G_2 \Pi^\top \|_F^2,
$$
where $\Pi_f$ is the permutation matrix corresponding to $\pi_f$. 
\end{algorithmic}
\end{algorithm}

Since one of the possible mapping $f$ will correspond to the underlying true matches of vertices in $I_0$, it follows that if our seeded graph matching succeeds with high probability and we are above the information-theoretic limit (so that the true matching minimizes the edge disagreements with high probability), the final output will coincide with the true matching with high probability, as stated in \prettyref{thm:seedless_graph_matching}.  More specifically, the proof is sketched below. 

\begin{proof}[Proof of \prettyref{thm:seedless_graph_matching}]
If $f: I_0 \to [n]$ is such that 
$f(i)=\pi^*(i)$ for all $i \in I_0$, then 
since our seeded graph matching succeeds with high probability, it follows that 
$
\pi_f = \pi^*
$
with high probability. 

Moreover, since we are above the information-theoretic limit, it follows from \cite[Theorem 1]{cullina2017exact} that with high probability, 
$$
\pi^* \in \arg \min_{\pi} 
\| G_1 - \Pi G_2 \Pi^\top \|_F^2,
$$
where $\Pi$ is the permutation matrix corresponding to $\pi$. 

Therefore, $\hat{\pi} =\pi^*$ with high probability.
Finally, since $\alpha n \to \infty$, it follows that $|I_0|$ is at most $2\alpha n$ with high probability. Hence, \prettyref{alg:graph_matching_seedless} calls the Seeded Graph Matching algorithm at most $n^{O(\alpha n)}$ times with high probability. 
\end{proof}


\section{Analysis of \prettyref{alg:graph_matching_ell_hop} in 
Sparse Graph Regime}\label{sec:sparse_simple}

In this and next two sections, we give the analysis of our algorithms and 
prove our main theorems. For the sake of
analysis, we assume 
$\pi^*=id$, \ie, $\pi^*(i)=i$
for all $i\in[n],$
without loss of generality.



Our analysis of \prettyref{alg:graph_matching_ell_hop} uses the  technique for analyzing small subgraph containment~\cite{Janson1997Random}. 
Let $T$ denote a starlike tree formed by $m$ independent (vertex-disjoint expect the root vertex) paths of length $\ell$ from root vertex to $m$ distinct leaves for $\ell,m \ge 1$. 
Note that $T$ consists of $m\ell+1$ vertices and $m\ell$ edges (See 
 \prettyref{fig:STexample} for an example of $m=3$ and $\ell=2$).  
Let $r(T)$ denote the root vertex of $T$ and $L(T)$ denote the set of leaves of $T.$
We say $T$ is a subgraph of $G$, denoted by $T \subset G$, 
if $V(T)\subset V(G)$ and $E(T) \subset E(G).$
The key  of our proof is to show that under certain conditions with high probability:
\begin{enumerate}
\item For every vertex $i$, there exists a copy of $T$ rooted at $i$ with all leaves seeded 
in the intersection graph $G_1^* \wedge G_2$;
\item There is no copy of $T_1 \cup T_2$ 
in  the union graph $G_1^* \vee G_2$. 
\end{enumerate}
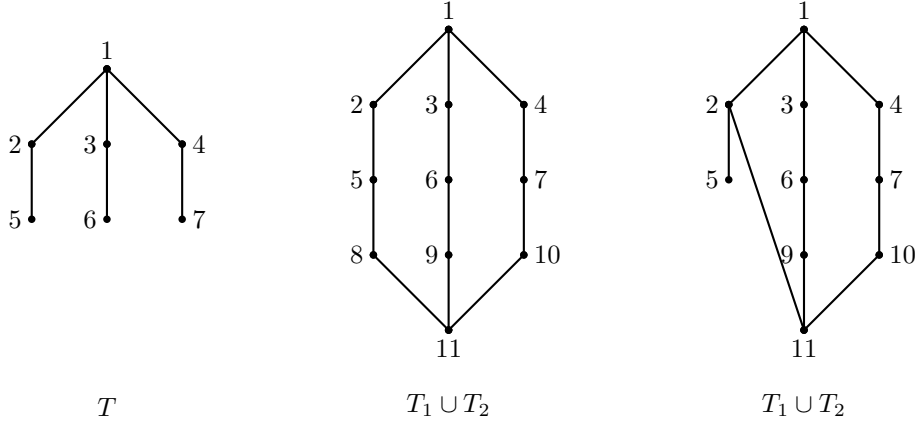
\begin{figure}[ht]%
\centering

\begin{tikzpicture}[scale=1,every edge/.append style = {thick,line cap=round},font=\small]

\node[coordinate,label=left:$6$] (v6) at (0,0-0.5)  {};
\node[coordinate,label=left:$5$] (v5) at (-1,0-0.5)  {};
\node[coordinate,label=right:$7$] (v7) at (1,0-0.5)  {};

\node[coordinate,label=left:$3$] (v3) at (0,1-0.5)  {};
\node[coordinate,label=left:$2$] (v2) at (-1,1-0.5)  {};
\node[coordinate,label=right:$4$] (v4) at (1,1-0.5)  {};

  \node[coordinate,label=above:$1$] (v1) at (0,2-0.5)  {};

 \singleedge{(v1)}{(v2)}
    \singleedge{(v1)}{(v3)}
    \singleedge{(v1)}{(v4)}

\singleedge{(v2)}{(v5)}
    \singleedge{(v3)}{(v6)}
    \singleedge{(v4)}{(v7)}
 
    \node at (0,-3) {$T$};
 \end{tikzpicture}
 \qquad \qquad
\begin{tikzpicture}[scale=1,every edge/.append style = {thick,line cap=round},font=\small]

\node[coordinate,label=left:$6$] (v6) at (0,0)  {};
\node[coordinate,label=left:$5$] (v5) at (-1,0)  {};
\node[coordinate,label=right:$7$] (v7) at (1,0)  {};

\node[coordinate,label=left:$3$] (v3) at (0,1)  {};
\node[coordinate,label=left:$2$] (v2) at (-1,1)  {};
\node[coordinate,label=right:$4$] (v4) at (1,1)  {};

  \node[coordinate,label=above:$1$] (v1) at (0,2)  {};

 \node[coordinate,label=left:$9$] (v9) at (0,-1)  {};
 \node[coordinate,label=left:$8$] (v8) at (-1,-1)  {};
 \node[coordinate,label=right:$10$] (v10) at (1,-1)  {};

  \node[coordinate,label=below:$11$] (v11) at (0,-2)  {};

 \singleedge{(v1)}{(v2)}
    \singleedge{(v1)}{(v3)}
    \singleedge{(v1)}{(v4)}

\singleedge{(v2)}{(v5)}
    \singleedge{(v3)}{(v6)}
    \singleedge{(v4)}{(v7)}

\singleedge{(v8)}{(v5)}
    \singleedge{(v9)}{(v6)}
    \singleedge{(v10)}{(v7)}

   \singleedge{(v11)}{(v8)}
    \singleedge{(v11)}{(v9)}
    \singleedge{(v11)}{(v10)}  
    \node at (0,-3) {$T_1 \cup T_2$};
 \end{tikzpicture}
 \qquad \qquad
\begin{tikzpicture}[scale=1,every edge/.append style = {thick,line cap=round},font=\small]

\node[coordinate,label=left:$6$] (v6) at (0,0)  {};
\node[coordinate,label=left:$5$] (v5) at (-1,0)  {};
\node[coordinate,label=right:$7$] (v7) at (1,0)  {};

\node[coordinate,label=left:$3$] (v3) at (0,1)  {};
\node[coordinate,label=left:$2$] (v2) at (-1,1)  {};
\node[coordinate,label=right:$4$] (v4) at (1,1)  {};

  \node[coordinate,label=above:$1$] (v1) at (0,2)  {};

\node[coordinate,label=left:$9$] (v9) at (0,-1)  {};
\node[coordinate,label=right:$10$] (v10) at (1,-1)  {};

 \node[coordinate,label=below:$11$] (v11) at (0,-2)  {};

 \singleedge{(v1)}{(v2)}
    \singleedge{(v1)}{(v3)}
    \singleedge{(v1)}{(v4)}

\singleedge{(v2)}{(v5)}
    \singleedge{(v3)}{(v6)}
    \singleedge{(v4)}{(v7)}

    \singleedge{(v9)}{(v6)}
    \singleedge{(v10)}{(v7)}

   \singleedge{(v11)}{(v2)}
    \singleedge{(v11)}{(v9)}
    \singleedge{(v11)}{(v10)}  
    \node at (0,-3) {$T_1 \cup T_2$};
 \end{tikzpicture}
\caption{Left: $T$ is a starlike tree with $m=3$, $\ell=2$, $r(T)=1$ and $L(T)=\{5,6,7\}.$
Middle and Right: Two examples of $T_1 \cup T_2$ such that $T_1, T_2$ are isomorphic to $T$, $r(T_1)\neq r(T_2)$, and $L(T_1)=L(T_2)=\{5,6,7\}$. 
For the middle, $V(T_1) \cap V(T_2)= \{5,6,7\}$; for the right, $V(T_1)\cap V(T_2) = \{ 2,5,6,7\}.$
}%
\label{fig:STexample}%
\end{figure}


\subsection{Success of \prettyref{alg:graph_matching_ell_hop} on the Intersection of Good Events}

We first introduce a sequence of ``good'' events on whose intersection, \prettyref{alg:graph_matching_ell_hop} 
correctly matches all vertices. We need the following graph properties:
\begin{enumerate}[(i)]
\item there is no isolated vertex;
  \item for any two adjacent vertices, there are at least $\tau$ vertices adjacent to
at least one of them;
\item For all vertices $i$ with $d_i \ge \tau $, 
there are at least $2m$ independent $\ell$-paths from $i$ to $2m$ distinct vertices in 
$I_0$;
\item There is no pairs of subgraphs $T_1, T_2 \subset G $ that are isomorphic to  $T$ 
such that $r(T_1)\neq r(T_2),$ and $L(T_1) = L(T_2)$ (See \prettyref{fig:STexample} for an illustration).
\item For every vertex $i$, there exist at most $m-1$ independent $\ell$-paths from $i$ 
to $m-1$ distinct vertices in $N^{G}_{\ell-1}(i)$. 
  \end{enumerate}

Let 
\begin{itemize}
\item $\calE_1$ denote the event such that $G_1^* \wedge G_2$ satisfy properties (i)--(iii);
\item $\calE_2$ denote the event such that $G_1^* \vee G_2$ satisfy properties (iv) and (v);
\item $\calE_3$ denote the event such that  
for any two vertices $i,j$ that are connected by a $2$-path in $G_1^* \vee G_2$,
  at least one  of the two vertices $i,j$ must be a high-degree vertex in $G_1^* \wedge G_2$.  
\end{itemize}

We claim that on event $\calE_1 \cap \calE_2 \cap \calE_3$, 
\prettyref{alg:graph_matching_ell_hop}
correctly matches all vertices. Recall that we can assume $\pi^*=id$ and thus $G_1=G^*_1$ without loss of generality. 

First, since $G_1^* \wedge G_2$ satisfy graph property (iii), 
it follows that in $G_1^* \wedge G_2$, for all high-degree vertices  $i$, 
there exist $2m$ independent $\ell$-paths to a set $S \subset \Gamma^{G_1^* \wedge G_2}_\ell (i)$ 
of $2m$ seeded vertices. 
Let $\tilde{S} = S \backslash N^{G^*_1 \vee G_2}_{\ell-1}(i)$. 
Since $G_1^* \vee G_2$ satisfy graph property (v), and 
$G_1^* \wedge G_2 \subset G_1^* \vee G_2$, it follows that 
$$
\left| S \cap N^{G^*_1 \vee G_2}_{\ell-1} (i) \right|  \le m-1
$$  
and thus 
$| \tilde{S} | \ge m+1 .$
Moreover, since $G_1^*\wedge G_2 \subset G_1, G_2$, it follows that 
$$
\tilde{S} \subset \Gamma^{G_1^* \wedge G_2}_\ell (i) 
\backslash N_{\ell-1}^{G_1^* \vee G_2}
\subset \Gamma^{G_1^*}_\ell (i) 
\cap \Gamma^{G_2}_\ell (i).
$$
Therefore, in both $G_1$ and $G_2$, there are 
 at least $m+1$ independent
$\ell$-paths from $i$ to $
\Gamma^{G_1^*}_\ell (i) 
\cap \Gamma^{G_2}_\ell (i).$

Second, note that on event $\calE_2$, $G_1^* \vee G_2$ satisfy graph property (iv). 
For the sake of contradiction, suppose there exist a pair of distinct vertices $i,j$
and a set $L$ of $m$ seeded vertices
such that there exist $m$ independent $\ell$-paths from $i$ to set $L$ in $G_1$
and $m$ independent $\ell$-paths from $j$ to set $L$ in $G_2$. 
Let $T_k$
denote the starlike tree
formed by the $m$ independent $\ell$-paths in $G_k$ for $k=1,2.$
Then $T_1,T_2 \subset G_1^* \vee G_2 $ are isomorphic to $T$ such that $r(T_1)=i, r(T_2)=j$ and 
$L(T_1)=L(T_2)=L$. This is in contradiction with the fact that 
$G_1^* \vee G_2$ satisfy graph property (iv). 

It follows from the above two points that \prettyref{alg:graph_matching_ell_hop} correctly matches all
high-degree vertices $i$ in $G_1^* \wedge G_2$, \ie, $\hat{\pi}(i)=i.$

Next, we show that all low-degree vertices are matched correctly. 
Fix a low-degree vertex $i$. Since $G_1^* \wedge G_2$ satisfy 
graph properties (i) and (ii), it must have a high-degree neighbor $j$ in $G_1^* \wedge G_2.$  Since 
the high-degree vertex $j$ has been matched 
correctly, $i$ is adjacent to $j$ in $G_1$ and $i$ is also adjacent to $\hat{\pi}(j)=j$ in $G_2.$
Moreover, for the sake of contradiction, suppose there exists a pair of two distinct low-degree vertices $i_1$ and $i_2$ such that
$i_1$ is adjacent to a matched vertex $j_1$ in $G_1$ and $i_2$ is adjacent to vertex $\hat{\pi}(j_1)$ in $G_2$.
Since $\hat{\pi}(j_1)=j_1$, it follows that $(i_1,j_1,i_2)$ form a $2$-path in $G_1^* \vee G_2.$ However, on event
$\calE_3$, $i_1$ and $i_2$ cannot be low-degree vertices simultaneously in $G_1^* \wedge G_2$, which leads to a contradiction. 
As a consequence, $\hat{\pi}(i)=i$ for every low-degree vertex $i.$

Finally, to prove \prettyref{thm:main}, it remains to show that under the theorem assumptions, 
$\prob{\calE_i} \to 1$ for all $i=1,2,3,$ which is done in the next subsection. 

\subsection{Bound the Probability of Good Events}
It is standard to prove that $G_1^* \wedge G_2$ satisfies properties (i)--(ii) with high probability and $\prob{\calE_3} \to 1$ using union bounds. For completeness, we state the 
lemmas and leave the proofs to appendices. 

\begin{lemma}\label{lmm:graph_properties}
Suppose $G \sim \calG(n,p)$ with $np - \log n \to +\infty$.
\begin{enumerate}[(i)]
  \item There is no isolated vertex in $G$ with probability at least $1-o(1)$;
\item Assume $\tau=o(np)$. With probability at least $1-n^{-1+o(1)}$,  for any two adjacent vertices, there are at least $\tau$ vertices adjacent to at least one of them in $G.$
\end{enumerate}
\end{lemma}

\medskip

\begin{lemma} \label{lmm:low_degree_path}
Assume 
$$
nps^2\ge \log n \quad \text{ and } \quad  \tau=o(nps^2) \quad \text{ and }  \quad
\log (np)=o(nps^2).
$$
With probability at least $1-n^{-1+o(1)}$, 
for any two vertices $i,j$ that are connected by a $2$-path in $G_1^* \vee
G_2$, at least one of the two vertices $i,j$ must have degree at least $\tau$ in $G_1^* \wedge G_2$. 
\end{lemma}

\medskip

It remains to show with high probability, $G_1^* \wedge G_2$ satisfy graph
property (iii) and $G_1^* \vee G_2$ satisfy graph properties (iv) and (v).

We will apply the following lemma to show that with high probability, 
for every high-degree vertex $i$ in $G_1^* \wedge G_2$, 
we can always
find at least $2m$ independent paths of length $\ell$ from 
$i$ to $2m$ distinct seeded vertices in $I_0.$

\begin{lemma}\label{lmm:existence_ell_path}
Suppose $G\sim \calG(n,p)$ and each vertex in $G$ is included in $I_0$ independently with probability $\alpha$.  
Assume  
$$
\alpha  (np/2)^{\ell-2} \tau (\tau-2m) - 2m \log \tau 
\ge 2\log n.
$$
and 
$$
p(4np)^{\ell} = o(1)
$$
and $\tau \to +\infty$. 
Then with high probability, for all vertices $i$ with $d_i \ge \tau$, there are at least $2m$ independent $\ell$-paths from $i$ to $I_0$.
\end{lemma}
\begin{proof}
In view of \prettyref{prop:tree_process}, we have
$\prob{\calH} \ge 1-3n^{-1+o(1)}$, where on event $\calH$, for every 
vertex $i$, there exists a tree $T_\ell(i) \subset G$ of depth $\ell$ rooted at $i$ such that:
\begin{enumerate}
\item 
Root $i$ has at most one children $j$ who has fewer than $\tau$ children in $T(i)$, \ie,  $|\Pi_1(j)| \le \tau$. 
\item For each children $j$ of $i$ with $|\Pi_1(j)| \ge \tau$, the subtree $T_{\ell-1}(j)$ 
of depth $\ell-1$ rooted at
$j$ has at least $\tau (np/2)^{\ell-2}$ leaves, \ie, $|\Pi_{\ell-1}(j)| \ge \tau \left( np/2 \right)^{\ell-2}$.
\end{enumerate}

Fix a vertex $i$ in $G$. Then $i$ has at least
$d_i -1$ children $j$ such that 
$|\Pi_1(j)| \le \tau$. For each such $j$, define
$Y_{ij}=1$ if there is a path of length $\ell-1$ from $j$ to some vertex in $I_0$ in $T_\ell(i)$. Then the number of independent paths from $i$ to $I_0$
is at least $
 \sum_{j=1}^{d_i-1} Y_{ij}.$
 
Since each leaf vertex of $T_{\ell-1}(j)$ is included in $I_0$  with probability $\alpha$
independently across different vertices and from  graph $G$, 
it follows that 
$$
\prob{  Y_{ij} =1 \mid  d_i \ge \tau, \calH } 
= 1- (1- \alpha)^{ |\Pi_{\ell-1}(j)| } 
\ge  1- \exp\left( - \alpha  \tau \left( np/2 \right)^{\ell-2} \right),
$$
where we used $1-x \le e^{-x}$ and $|\Pi_{\ell-1}(j)| \ge \tau \left( np/2 \right)^{\ell-2}$ on event $\calH$.  
Therefore, 
 \begin{align*}
 \prob{ \sum_{j=1}^{d_i-1} Y_{ij} \le 2m-1 | d_i \ge \tau, \calH} & \le
\prob{ \sum_{j=1}^{\tau-1} Y_{ij} \le 2m-1 } \\
 & \le \prob{ \Binom \left( \tau -1,  1 - e^{- \alpha \tau  \left( np/2 \right)^{\ell-2} } \right)
 \le 2m-1 } \\
 & = \sum_{k=0}^{2m-1} \binom{\tau-1}{k} e^{- \alpha \tau \left( np/2 \right)^{\ell-2} (\tau-1-k ) }  \\
 & \le e^{- \alpha \tau \left( np/2 \right)^{\ell-2} (\tau-2m)}
 \sum_{k=0}^{m-1} \tau^k \\
 & \le 2 e^{- \alpha \tau \left( np/2 \right)^{\ell-2} (\tau-2m)}
 \tau^{2m}
 \le 2 n^{-2},
 \end{align*}
 where the last equality holds due to 
 the assumption $\alpha(\tau-2m) \left( np/2 \right)^{\ell-2} - 2m \log \tau \ge 2\log n.$

Define event 
$$
\calF_i =\left\{  d_i \ge \tau \right\} \cap\left\{ \sum_{j=1}^{d_i-1} Y_{ij}  \le 2m -1 \right\}.
$$
Then we have that 
$$
\prob{\calF_i \cap \calH } 
\le \prob{ \sum_{j=1}^{d_i-1} Y_{ij} \le 2m-1 | d_i \ge \tau, \calH}
\le 2 n^{-2}.
$$

Let $\calF=\cup_i \calF_i$. 
By the union bound, 
$$
\prob{\calF} = \prob{\calF \cap \calH} + \prob{\calH^c}
\le \sum_i \prob{\calF_i \cap \calH} + \prob{\calH^c}
\le 2n^{-1} + 3n^{-1+o(1)} \le 5n^{-1+o(1)}. 
$$
Therefore, with high probability, for all vertices $i$ with $d_i \ge \tau$, 
$\sum_{j=1}^{d_i-1} Y_{ij} \ge 2m$. 


\end{proof}

The following lemma will be useful to conclude that in $G_1^* \vee G_2$, 
with high probability, there is no pair of subgraphs $T_1, T_2
\subset G_1^*\vee G_2$ that are isomorphic to $T$,
such that $r(T_1)\neq r(T_2)$ and $L(T_1)=L(T_2).$
See \prettyref{fig:STexample} for an illustration of 
$T_1$ and $T_2$ isomorphic to $T$ such that 
$r(T_1)\neq r(T_2)$ and $L(T_1)=L(T_2).$

\begin{lemma}\label{lmm:subgraph_count_general}
Suppose $G \sim \calG(n,p)$ and $\ell,m \ge 1$. Then it holds that 
\begin{align*}
& \prob{ \exists\; T_1, T_2 \subset G \text{ that are isomorphic to } T: r(T_1)\neq r(T_2), L(T_1) = L(T_2) 
} \\
& \le \left( 2+ \frac{8}{np} \right)^{m(\ell-1)}  n^{2m\ell+2-m} p^{2m\ell}. 
\end{align*}
\end{lemma}
\begin{proof}
Let $\calT$ denote the set of all possible subgraphs that are isomorphic to $T$ in the complete graph $K_n$.  
By the union bound, we have
\begin{align*}
& \prob{ \exists\; T_1, T_2 \subset G \text{ that are isomorphic to } T: r(T_1)\neq r(T_2), L(T_1) = L(T_2) 
} \\
& \le \sum_{T_1, T_2 \in \calT: \; r(T_1)\neq r(T_2), L(T_1) = L(T_2) } \prob{ T_1, T_2 \subset G }
\end{align*}
For each such pair of $T_1, T_2$, 
$$
\prob{T_1, T_2 \subset G} = p^{ |E(T_1)| + |E(T_2)| - |E(T_1 \cap T_2)| } = p^{2m\ell-|E(T_1 \cap T_2)|},
$$
where the last equality holds because $T_1$ and $T_2$ are isomorphic to $T$ and $|E(T)|=2m\ell.$

Next for any given unlabelled graph $S$, we enumerate 
all the possible distinct pairs of $T_1, T_2 \in \calT$ such that $T_1 \cap T_2$ is isomorphic to $S$,
$r(T_1)\neq r(T_2)$, and $L(T_1) = L(T_2)$.
Let $\kappa_S$ denote the number of subgraphs $S'$ in $T$ such that $S'$ is isomorphic to $S$,
$L(T) \subset V(S')$, and $r(T) \notin V(S').$
Then there 
are at most $\kappa_S^2$ ways of intersecting $T_1$ and $T_2$ such that $T_1 \cap T_2$ is isomorphic to $S$,
$r(T_1)\neq r(T_2)$, and $L(T_1) = L(T_2)$.
For each such type of intersection, there are  at most
$n^{|V(S)|}$ different choices for vertex labelings of $T_1\cap T_2$, and 
$n^{2 \left( |V(T)| - |V(S)| \right) }$ different choices for vertex labelings of
$(T_1 \backslash T_2) \cup (T_2 \backslash T_1)$. 
Hence, the total number of 
distinct pairs of $T_1, T_2 \in \calT$ such that $T_1 \cap T_2$ is isomorphic to $S,$
$r(T_1)\neq r(T_2)$, and $L(T_1) = L(T_2)$ 
is at most 
$$
\kappa^2_S n^{|V(S)| } n^{2 \left( |V(T)| - |V(S)| \right) } = \kappa^2_S  
n^{2m\ell+2-|V(S)|},
$$
where the last equality holds due to $|V(T)|=m\ell+1$.

Combining the last two displayed equations yields that 
\begin{align*}
 \sum_{T_1, T_2 \in \calT: \; r(T_1)\neq r(T_2), L(T_1) = L(T_2) } \prob{ T_1, T_2 \subset G } 
 \le \sum_{S}  \kappa_S^2 n^{2m\ell+2-|V(S)|} p^{2m\ell-|E(S)|}.
\end{align*}
Note that if $\kappa_S \ge 1$, then by the definition of $\kappa_S$, 
$S$ is isomorphic to some $S' \subset T$
such that $L(T) \subset V(S')$ and  $r(T)\notin V(S')$.
By the starlike tree property of $T$, $S'$ is 
a forest with at least $m$ disjoint trees; hence 
so is $S$. See \prettyref{fig:STexample} for two 
illustrating examples. Therefore, 
$$
E(S) \le V(S) -m. 
$$
Hence,
\begin{align*}
\sum_{S} \kappa_S^2 n^{2m\ell+2-|V(S)|} p^{2m\ell-|E(S)|}  \le \sum_{S  }  \kappa_S^2  n^{2m\ell+2-|V(S)|} p^{2m\ell+m-|V(S)|}.
\end{align*}
Finally, we break the summation in the right hand side of the last displayed equation 
according to $|V(S)|$. 
In particular, let $|V(S)|=m+k$
for $0 \le k \le m(\ell-1)$. Note that 
$
\sum_{S} \kappa_S
$
is at most the number of distinct subgraphs $S'$ of $T$
such that $L(T) \subset V(S'), r(T) \notin V(S')$ and $|V(S')|=m+k$, 
which is further upper bounded by
$\binom{m(\ell-1)}{k} 2^{k}$,
because there
are at most $\binom{m(\ell-1)}{k}$ different choices for $V(S')\backslash L(T)$
and at most $2^{|V(S')|-m}$ choices for determining whether 
to include the edges induced by $V(S')$ in $T$ into $S'$. 
Hence, 
\begin{align*}
& \sum_{S} \kappa_S^2 
n^{2m\ell+2-|V(S)|} p^{2m\ell+m-|V(S)|}  \\
 &  =  \sum_{k=0}^{m(\ell-1)}  n^{2m\ell+2-m-k} p^{2m\ell-k} 
\sum_{S: |V(S)|=m+k  } \kappa_S^2 \\
 &  \overset{(a)}{\le} \sum_{k=0}^{m(\ell-1)}  n^{2m\ell+2-m-k}p^{2m\ell-k} \left( \binom{m(\ell-1)}{k} 2^{k} \right)^2  \\
  &   \overset{(b)}{\le} n^{2m\ell+2-m} p^{2m\ell} 2^{m(\ell-1)} 
  \sum_{k=0}^{m(\ell-1)}  n^{-k} p^{-k}  \binom{m(\ell-1)}{k}  4^{k} \\
  & = n^{2m\ell+2-m} p^{2m\ell} 2^{m(\ell-1)}  \left( 1+ \frac{4}{np} \right)^{m(\ell-1) },  
\end{align*}
where $(a)$ follows from $\sum_{S} \kappa_S \le \binom{m(\ell-1)}{k} 2^{k}$,
and $(b)$ holds due to $ \binom{m(\ell-1)}{k} \le 2^{m(\ell-1)}.$
\end{proof}

Finally, we need a result to conclude that with high probability, 
for every vertex $i$, there exist at most 
$m-1$ independent $\ell$-paths from $i$ to $m-1$ distinct vertices in
$N^{G_1^* \vee G_2}_{\ell-1}(i)$.

Fix $m, \ell \ge 1$. We start with any vertex $i$ and $m$ 
independent (vertex-disjoint except for $i$)
paths of length $\ell$ 
from $i$ to $m$ distinct vertices $j_1, \ldots, j_m$, denoted by $P_1, \ldots, P_m$. 
Let $\tilde{P}_k$ denote any path of length at most $\ell-1$ from $i$ to $j_k$
for $k=1, \ldots, m$. Let $H=\cup_{k=1}^m ( P_k \cup \tilde{P}_k )$ 
and $\calH_{m,\ell}$ denote the family of all possible graphs $H$ with $V(H)\subset [n]$
obtained by the above procedure. 

Note that if there is no subgraph isomorphic to some $H \in \calH_{m,\ell}$ in $G_1^* \vee G_2$,
then for every vertex $i$, there exist at most 
$m-1$ independent $\ell$-paths from $i$ to $m-1$ distinct vertices in
$N^{G_1^* \vee G_2}_{\ell-1}(i)$. 
Hence, our task reduces to proving that with high probability, $G_1^* \vee G_2$ does not contain 
some $H \in \calH_{m,\ell}$ as a subgraph.

We first need a lemma showing that any $H \in \calH_{m,\ell}$ is so ``dense'' that 
it appears as a subgraph in $\calG(n,p)$ with a vanishing small probability. 

\begin{lemma}\label{lmm:e_v_diff}
Fix $m,\ell \ge 1$. For any $H \in \calH_{m,\ell}$, 
$$
| E(H)|  \ge  |V(H)| + m -1.
$$
 \end{lemma}
\begin{proof}
Recall that $H=\cup_{k=1}^m ( P_k \cup \tilde{P}_k )$, where 
$P_1, \ldots, P_m$ is a set of $m $ (vertex-disjoint except for $i$)
paths of length $\ell$ 
from $i$ to $m$ distinct vertices $j_1, \ldots, j_m$, and 
$\tilde{P}_k$ is a path of length at most $\ell-1$ from $i$ to $j_k$
for $k=1, \ldots, m$. 

Note that we order the vertices and edges in paths starting from $i$. 
For each $k=1, \ldots, m$, let $v_k$ denote the
 first vertex after which $P_k$ and $\tilde{P}_k$ completely conincide, 
and $e_k$ denote the edge incident to $v_k$ 
in $\tilde{P}_k$. Then by definition, $v_k \neq i$ and
$e_k \in \tilde{P}_k \backslash P_k$. 
Let $\text{dist}(u,v)$ denote the 
\emph{longest} distance 
between $u$ and $v$ in $H$,
and $\sigma$ denote any permutation on $[m]$ such that 
$$
\text{dist} \left(i, v_{\sigma(1)}  \right) \ge  \text{dist}\left(i, v_{\sigma(2)} \right)
\cdots \ge \text{dist} \left(i, v_{\sigma(m)} \right).
$$
Without loss of generality, we assume $\sigma=id$, \ie, $\sigma(k)=k$. 
We claim that $e_{j} \notin P_k \cup \tilde{P_k}$ for any 
$1 \le j < k \le m.$ 
In fact, $e_j \notin P_k$, because otherwise $P_j$ and $P_k$ share 
a common vertex $v_j \neq i$, which violates the assumption that 
$P_j$ and $P_k$ are vertex-disjoint except for $i.$
Also, $e_j \notin \tilde{P}_k \backslash P_k$, because otherwise, 
$e_j$ is ordered before $e_k$ in path $\tilde{P}_k$ starting
from $i$, which implies $\text{dist}(i,v_k) > \text{dist}(i,v_j)$
and leads to a contradiction.

Finally, we recursively define $H_0=H$ and $H_{k}$ such that
$V(H_k)= V( H_{k-1})$ and $E(H_k) = E(H_{k-1}) \backslash \{e_k\}$
for $k=1, \ldots, m$. We prove that $H_m$
is connected by induction. For the base case $k=0,$ clearly $H_0=H$
is connected. Suppose $H_{k-1} $ is connected. Since we have shown
that $e_{j} \notin P_k \cup \tilde{P_k}$ for any 
$1 \le j < k \le m$, it follows that 
$ P_k \cup \tilde{P_k} \subset H_{k-1}$.
Note that there is a path through $i$ 
between the two endpoints of $e_k$ in 
$P_k \cup \tilde{P_k}$. Hence, the
two endpoints of $e_k$ are still connected in
$H_k$. Moreover, by the induction hypothesis, $H_{k-1}$
is connected. Therefore, $H_k$ is connected. 
and it follows from induction that $H_m$ is connected. 
Thus, 
$
 | E(H_m) | - | V(H_m) | \ge -1. 
$
Since $| E(H) |= |E(H_m)| +m$ and $ | V(H)| = |V(H_m)|$,
it follows that 
$|E(H)|-|V(H)| \ge m-1.$

\end{proof}

Next we state a lemma which upper bounds the number of isomorphism classes in
$\calH_{m,\ell}.$ This upper bound is by no means tight, but suffices for our purpose. 

\begin{lemma}\label{lmm:H_count}
Fix $m,\ell \ge 1$. 
Denote by $\calU_{m,\ell}$ the set of unlabelled graphs (isomorphism classes) in 
$\calH_{m,\ell}$. Then 
$$
\left| \calU_{m,\ell}  \right| \le  (3\ell)^m 3^{2m^2\ell}.
$$
\end{lemma}
\begin{proof}
Recall that $H=\cup_{k=1}^m ( P_k \cup \tilde{P}_k )$, where 
$P_1, \ldots, P_m$ is a set of $m $ (vertex-disjoint except for $i$)
paths of length $\ell$ 
from $i$ to $m$ distinct vertices $j_1, \ldots, j_m$, and 
$\tilde{P}_k$ is a path of length at most $\ell-1$ from $i$ to $j_k$
for $k=1, \ldots, m$. 
Let $T = \cup_{k=1}^m P_k$. Then $T$ is a starlike tree rooted at $i$
with $m$ branches as depicted in \prettyref{fig:STexample}.

We fix a sequence of $\{ \ell_1, \ldots, \ell_m\}$ with $1 \le \ell_k \le \ell-1$.
Let $\calU_{\ell_1,\ldots, \ell_m}$ denote all the possible 
unlabelled graphs formed by the union of $T$ and 
$\tilde{P}_k$ of length $\ell_k$ for $k \in [m]$.  
For ease of notation, let $\tilde{P}_0=T$. 
We enumerate $\calP_{\ell_1,\ldots, \ell_m}$ according to the pairwise intersections
 $\tilde{P}_j \cap \tilde{P}_k$ for 
$0 \le j < k \le m$.
Specifically, for any given sequence 
 $\{S_{jk}: 0 \le j <k \le m \}$ 
of unlabelled graphs, we enumerate
 all the possible sequences of $(\tilde{P}_1, \ldots, \tilde{P}_k)$ 
 such that $\tilde{P}_j \cap \tilde{P}_k$ is isomorphic to $S_{jk}$ 
 for $0 \le j < k \le m$. 
Let $\kappa_\ell(S)$ denote the number of possible
 different subgraphs that are isomorphic to $S$
 in an $\ell$-path.
 Recall $\beta(S)$ denote the number of possible different
 subgraphs that are isomorphic to $S$ in $\tilde{P}_0=T$. 

Then across all $1 \le j<k \in [m]$, there are at most 
 $\kappa_{\ell_j}(S) \kappa_{\ell_k}(S)$ ways of 
 intersecting $\tilde{P}_j$ and $\tilde{P}_k$
 such that $\tilde{P}_j \cap \tilde{P}_k$ is isomorphic to $S$. 
 Also, for all $k \in [m]$, there are at most 
 $\beta(S) \kappa_{\ell_k}(S)$ ways of 
 intersecting $\tilde{P}_0$ and $\tilde{P}_k$
 such that $\tilde{P}_0 \cap \tilde{P}_k$ is isomorphic to $S$.
 Hence, the total number of 
 distinct sequences of $(\tilde{P}_1, \ldots, \tilde{P}_k)$ 
 such that $\tilde{P}_j \cap \tilde{P}_k$ isat most the
 the number $n_\ell$ of distinct subgraphs in an $\ell$-path isomorphic to $S_{jk}$ 
 for $0 \le j <k \le m$ is at most 
 $$
 \prod_{1 \le j<k \in [m]} \kappa_{\ell_j} (S_{jk} ) \kappa_{\ell_k} (S_{jk})
 \prod_{ k \in [m]} \beta(S_{0k}) \kappa_{\ell_k}(S_{0k}).
 $$
 Therefore,
 \begin{align*}
 \left|\calU_{\ell_1,\ldots, \ell_m} \right| & \le 
 \sum_{ \{ S_{jk}: 0 \le j < k \le m \}} 
 \prod_{j<k \in [m]} \kappa_{\ell_j} (S_{jk} ) \kappa_{\ell_k} (S_{jk}) \prod_{ k \in [m]} \beta(S_{0k}) \kappa_{\ell_k}(S_{0k}) \\
 & \le  \prod_{ 1 \le j<k \in [m] } \left( \sum_{S_{jk}}   \kappa_{\ell_j} (S_{jk} )  \right) 
  \left( \sum_{S_{jk}} \kappa_{\ell_k} (S_{jk}) \right) 
  \prod_{ k \in [m] } \left( \sum_{S_{0k}}  \beta (S_{0k}) \right) \left( \sum_{S_{0k}} \kappa_{\ell_k} (S_{0k})  \right)
  \\
  & \le \prod_{j<k \in [m] } n_{\ell_j} n_{\ell_k}  \prod_{ k \in [m] } n(T) n_{\ell_k}
  =  \left( n(T) \right)^{m} \prod_{ k \in [m] }  \left( n_{\ell_k} \right)^{m} ,
 \end{align*}
 where the last inequality holds because  
 $\sum_{S} \kappa_{\ell} (S)$ is at most the
 the number $n_\ell$ of distinct subgraphs $S'$ in an $\ell$-path, 
 and $\sum_{S} \beta (S)$ is at most the
 the number $n(T)$ of distinct subgraphs $S'$ in $T.$
 Note that 
 $$
 n_\ell \le \sum_{k=0}^{\ell} \binom{\ell}{k} 2^{k}
 =3^{\ell},
 $$
 because if $|V(S')|=k$, then 
 there are at most $\binom{\ell}{ k}$ different
 choices for $V(S')$ and at most $2^{k}$ choices for 
 determining whether to include the edges induced by $V(S')$
 in an $\ell$-path into $S'.$ 
 Also,
 $$
 n(T) \le \sum_{k=0}^{m\ell+1} \binom{m\ell+1}{k} 2^{k} = 3^{m\ell+1}. 
 $$
 Combinining the last three displayed equations yields that
 $$
 \left|\calU_{\ell_1,\ldots, \ell_m} \right| \le 3^{2m^2\ell + m}.
 $$
Therefore,
$$
\left|\calU \right| = \sum_{ (\ell_1, \ldots, \ell_m):
1 \le \ell_k \le \ell-1}
 \left|\calU_{\ell_1,\ldots, \ell_m} \right|
 \le (3\ell)^m 3^{2m^2\ell}.
$$

\end{proof}

With \prettyref{lmm:e_v_diff} and \prettyref{lmm:H_count}, we are ready to bound
the probability that $\calG(n,p)$ contains some $H \in \calH_{m,\ell}$ as a subgraph. 
\begin{lemma}\label{lmm:cycle_bound}
Suppose $G\sim \calG(n,p)$ with $np \ge 1$ and $m,\ell \ge 1$. Then it holds that 
\begin{align}
\prob{ \exists H \in \calH_{m,\ell}: H \subset G }  \le 
n^{2m\ell-2m+1} p^{ 2m\ell-m} (3\ell)^m 3^{2m^2\ell}.
\end{align}
\begin{proof}
Note that for any $H \in \calH_{m,\ell}$, 
$$
 m \ell+1 \le | V(H)|  \le m\ell+1 + (\ell-2) m = 2m\ell - 2m+1,
$$
where the lower bound holds because $H$ contains a starlike tree with $m\ell+1$ distinct
vertices, and the upper bound holds when $P_k$ and $\tilde{P}_k$ are all vertex-disjoint
except for the source vertex and sink vertices. 

For any given integer $ m \ell +1 \le t \le 2m\ell -2m+1$, define
$$
 \calH_{m,\ell,t} = \left\{ H \in \calH_{m,\ell}: V(H) \subset [n], |V(H)| =t \right\}. 
 $$
and let $\calU_{m,\ell_t}$ denote the number of unlabelled graphs (isomorphism class) in 
$\calH_{m,\ell}$. 
Since $V(H) \subset [n]$ and $|V(H)| =t $, there are at most $n^{t}$ different
vertex labelings for a given unlabelled graph $U \in \calU_{m,\ell,t}$. Hence,
\begin{align}
\left|  \calH_{m,\ell,t} \right| \le \left| \calU_{m,\ell,t} \right| n^t. 
\label{eq:isorm_count}
\end{align}
By the union bound, we have
\begin{align*}
\prob{ \exists  H \in \calH_{m,\ell}: H \subset G } & \le \sum_{t=m\ell+1}^{2m\ell-2m+1}  \sum_{H \in \calH_{m,\ell,t}} \prob{ H \subset G} \\
& = \sum_{t=m\ell+1}^{2m\ell-2m+1} \sum_{H \in \calH_{m,\ell,t} } p^{ | E(H) | } \\
& \overset{(a)}{\le} \sum_{t=m\ell+1}^{2m\ell-2m+1} \sum_{H \in \calH_{m,\ell,t}} p^{ t + m-1} \\
& \overset{(b)}{\le} \sum_{t=m\ell+1}^{2m\ell-2m+1}  n^{t} p^{ t + m-1} \left| \calU_{m,\ell,t} \right|   \\
& \overset{(c)}{\le} n^{2m\ell-2m+1} p^{ 2m\ell-m} \left| \calU_{m,\ell} \right| \\
& \le n^{2m\ell-2m+1} p^{ 2m\ell-m} (3\ell)^m 3^{2m^2\ell},
\end{align*}
where $(a)$ holds in view of \prettyref{lmm:e_v_diff};
$(b)$ holds in view of \prettyref{eq:isorm_count}
$(c)$ holds because $np \ge 1$ and $t \le 2m\ell-2m+1$;
the last inequality holds due to \prettyref{lmm:H_count}.
\end{proof}
\end{lemma}

\subsection{Completing the Proof of \prettyref{thm:main} }

Recall that the choices of 
$\ell$ in \prettyref{eq:def_ell_1}, $m$ in
\prettyref{eq:def_m}, 
and $\tau$ in \prettyref{eq:def_tau}. In particular,
$$
\left( nps^2 \right)^{\ell} \le n^{1/2-\epsilon}. 
$$

Recall that $G_1^* \wedge G_2 \sim \calG(n,ps^2)$. Under the assumption that $\alpha \ge n^{-1/2+3\epsilon}$, we get that
$$
\alpha (nps^2/2)^{\ell-2} \tau (\tau-m) - m \log \tau \ge 2\log n.
$$
Hence, applying \prettyref{lmm:existence_ell_path}, we conclude that $G_1^* \wedge G_2$ satisfy graph property (iii). Combing this result with \prettyref{lmm:graph_properties}, we get that 
$\prob{\calE_1} \ge 1-o(1)$.

Note that $G_1^* \vee G_2 \sim \calG(n, ps(2-s))$. 
We first apply \prettyref{lmm:subgraph_count_general} to $G_1^* \vee G_2$. 
In view of $nps^2 \ge \log n$ and $n \ge e$, we get that $nps(2-s)\ge \log n \ge 1$ and 
thus
\begin{align*}
& \left( 2+ \frac{8}{nps(2-s)} \right)^{m(\ell-1) } n^{2m\ell+2-m} \left( ps (2-s) \right)^{2m\ell}  \\
& \le 10^{m\ell} n^{2-2\epsilon m } \left( \frac{2-s}{s}\right)^{2ml} 
=n^{-2+o(1)}, 
\end{align*}
where the first inequality holds due to $(nps^2)^\ell \le n^{1/2-\epsilon}$; the last equality holds by our choice of $\ell$ and $m$ and $s=\Theta(1)$. 
Hence, applying \prettyref{lmm:subgraph_count_general} to $G_1^* \vee G_2$, we conclude that 
with high probability, there is no pair of subgraphs $T_1, T_2 \subset G_1^* \vee G_2$ that
are isomorphic to $T$ such that $r(T_1)\neq r(T_2)$ and $L(T_1)=L(T_2).$

Then we apply \prettyref{lmm:cycle_bound} to $G_1^* \vee G_2$. Note that 
\begin{align*}
& n^{2m\ell-2m+1} \left( p s (2-s) \right)^{ 2m\ell-m} (3\ell)^m 3^{2m^2\ell} \\
& \le  n^{m (1-2\epsilon) - m+1} (np)^{-m} \left( \frac{2-s}{s} \right)^{2m\ell -m} (3\ell)^m 3^{2m^2\ell}  \le n^{ -3 + o(1)},
\end{align*}
where the first inequality holds due to $(nps^2)^\ell \le n^{1/2-\epsilon}$;
the last equality holds by our choice of $\ell$ and $m$ and $s=\Theta(1)$. 
Hence, applying \prettyref{lmm:cycle_bound} to $G_1^* \vee G_2$,
we conclude that with high probability, $G_1^* \vee G_2$ does not contain 
any graph $H \in \calH_{m,\ell}$ as a subgraph. By the construction of 
$\calH_{m,\ell}$, it further implies that with high probability, for every
vertex $i$, there exist at most $m-1$ independent paths from $i$
to $m-1$ distinct vertices in $N^{G_1^* \vee G_2}_{\ell-1}$. 

Combining the above two points, we get that $\prob{\calE_2} \to 1$.  
Finally, in view of \prettyref{lmm:low_degree_path}, we get that 
$\prob{\calE_3} \ge 1-o(1)$, completing the proof of \prettyref{thm:main}.

\section{Analysis of \prettyref{alg:graph_matching_ell_hop_witness} in  Dense Graph Regime}

Recall that $\Gamma^k_G(u)$ and $N^G_k(u)$ denotes the set of vertices \emph{at} and \emph{within} distance $k$ from $u$ in graph $G$, respectively, as defined in \prettyref{eq:def_Gamma_neighbor} and \prettyref{eq:def_N_neighbor}. The key is to show that $| N^{G_1^* \wedge G_2}_{d-1} (u) |$ is larger than $| N^{G_1^* \vee G_2}_{d-1} (u)  \cap N^{G_1^* \vee G_2}_{d-1} (v)| $ for $u \neq v$ by a constant factor, so that we can matches two vertices correctly based on the number of common seeded vertices in their two large neighborhoods. 

\begin{proof}[Proof of \prettyref{thm:dense}]
Define event 
$$
\calA = 
\left\{ \left| N^{G_1^* \wedge G_2}_{d-1} (u) \right|  \ge \frac{3}{4} (nps^2)^{d-1} , \; \forall u \right\}.
$$
In view of claim (i) in \prettyref{lmm:bollobas_neighbor_expansion} with
$G= G_1^* \wedge G_2$ and
the fact that $\Gamma_k^G(u) \subset N_k^G(u)$, we get that $\prob{\calA} \ge 1-n^{-10}$.

Define event
$$
\calB = 
\left\{ \left| N^{G_1^* \vee G_2}_{d-1} (u)  \cap N^{G_1^* \vee G_2}_{d-1} (v)  \right| \le \frac{1}{2} (nps^2)^{d-1}, \; \forall u \neq v \right\}. 
$$
Note that due to assumption \prettyref{eq:assumption_b}, 
$$
\frac{1}{2} (nps^2)^{d-1}  \ge 8 n^{2d-3} \left( ps(2-s) \right)^{2d-2}.
$$
Hence, applying claim (ii) in \prettyref{lmm:bollobas_neighbor_expansion} with
$G= G_1^* \vee G_2$, we get that $\prob{\calB} \ge 1-n^{-10}$. 

Recall that $I_0$ is the initial set of seeded vertices. 
Define event 
$$
\calC = \left\{ \left| N^{G_1^* \wedge G_2}_{d-1} (u) \cap I_0 \right|  > \frac{3}{5} (nps^2)^{d-1} \alpha , \; \forall u \right\}.
$$
Since each vertex is seeded independently with probability $\alpha$, 
it follows that 
\begin{align*}
\prob{ \calC^c } & \le \prob{\calA^c} + \prob{ \calC^c \mid \calA} \\
& \le n^{-10} + \sum_{u} \prob{ \left| N^{G_1^* \wedge G_2}_{d-1} (u) \cap I_0 \right|  \le \frac{3}{5} (nps^2)^{d-1} \alpha \mid \calA } \\
&  \le n^{-10} + n \prob{ \Bin \left( 
\left \lceil \frac{3}{4} (nps^2)^{d-1} \right \rceil, \alpha \right) \le \frac{3}{5} (nps^2)^{d-1} \alpha } \\
& \le n^{-10} + n \exp \left( - \frac{3}{200} (nps^2)^{d-1} \alpha  \right) \le 2 n^{-1},
\end{align*}
where the last inequality holds due to assumption \prettyref{eq:assumption_alpha_dense}. 

Similarly, define event
$$
\calD = 
\left\{ \left| N^{G_1^* \vee G_2}_{d-1} (u)  \cap N^{G_1^* \vee G_2}_{d-1} (v) \cap I_0  \right| < \frac{3}{5} (nps^2)^{d-1} \alpha, \; \forall u \neq v \right\}. 
$$
It follows that 
\begin{align*}
\prob{ \calD^c } & \le \prob{\calB^c} + \prob{ \calD^c \mid \calB} \\
& \le n^{-10} + \sum_{u} \prob{ \left| N^{G_1^* \vee G_2}_{d-1} (u)  \cap N^{G_1^* \vee G_2}_{d-1} (v) \cap I_0  \right|  \ge \frac{3}{5} (nps^2)^{d-1} \alpha \mid \calB } \\
&  \le n^{-10} + n \prob{ \Bin \left( 
\left \lceil \frac{1}{2} (nps^2)^{d-1} \right \rceil, \alpha \right) \le \frac{3}{5} (nps^2)^{d-1} \alpha } \\
& \le n^{-10} + n \exp \left( - \frac{1}{150} (nps^2)^{d-1} \alpha  \right) \le 2n^{-1}, 
\end{align*}
where the last inequality holds again due to assumption \prettyref{eq:assumption_alpha_dense}. 
Hence, $\prob{\calC \cap \calD} \ge 1- 4n^{-1}$. 

Finally, since $G^*_1\wedge G_2$ is 
a subgraph of both $G^*_1$ and $G_2$,
it follows that 
$$
N^{G_1^* \wedge G_2}_{d-1} (i_2)
\subset 
\left\{ j \in I_0 : \pi_0(j) \in N_{d-1}^{G_1} \left(\pi^*(i_2) \right), j \in N^{G_2}_{d-1} (i_2)  \right\}.
$$
Similarly, both $G^*_1$ and $G_2$
are subgraphs of $G_1^* \vee G_2$,
it follows that 
$$
\left\{ j \in I_0 : \pi_0(j) \in N_{d-1}^{G_1} \left( i_1 \right) , j \in N^{G_2}_{d-1} (i_2)  \right\}
\subset 
N^{G_1^* \vee G_2}_{d-1} \left( (\pi^*)^{-1}(i_1) \right)  \cap N^{G_1^* \vee G_2}_{d-1} (i_2).
$$
Thus, on event
$\calC \cap \calD$, for every vertex $i_2 \in V(G_2) \setminus I_0$, 
\begin{align*}
w_{i_1 , i_2}
\begin{cases}
   > \frac{3}{5} (nps^2)^{d-1} \alpha & \text{ if } i_1 = \pi^*(i_2) \\
   < \frac{3}{5} (nps^2)^{d-1} \alpha & \text{ o.w. }.
\end{cases}
\end{align*}
Hence, \prettyref{alg:graph_matching_ell_hop_witness}
outputs $\hat{\pi} =\pi^*$ on event $\calC \cap \calD$. 
\end{proof}

\section{Analysis of \prettyref{alg:graph_matching_ell_hop_witness_sparse} in Sparse Graph Regime}

Recall that we assume $\pi^*=id$ without loss of generality in the analysis. Before proving \prettyref{thm:sparse_improved}, we present two key lemmas.

The first lemma will be used later to conclude that 
the test statistic $Z_{u,u}$ given in \prettyref{eq:Z_sparse} is large for all high degree vertices $u$.  

\begin{lemma}\label{lmm:signal_sparse}
Suppose $G \sim \calG(n,p)$ with $\log n \le np \le n^{\epsilon}$,
and each vertex is included in $I_0$ with probability $\alpha$. 
Recall that $\ell$ and $\eta$ are given in \prettyref{eq:def_ell}
and \prettyref{eq:def_eta}, respectively. Assume $\eta \ge 4 \log n$.  Let $G\backslash S$ denote the graph $G$ with set of vertices $S$ removed. 
Then with probability at least $1-n^{-1+o(1)}$,
$$
\sum_{j \in \Gamma_1^G(i) }  \indc{ | \Gamma_{\ell}^{ G \backslash S } (j) 
\cap I_0 | \ge \eta } 
\ge d_i -1, \quad \forall S \text{ s.t. } 
i\in S, |S| \le 3. 
$$
\end{lemma}
\begin{proof}
For every vertex $i$ and its neighbor $j \in \Gamma_1^G(i)$, define 
$$
a_{ij}  = \indc{ | \Gamma_{\ell}^{ G\backslash S } (j) | \ge 4 \eta  /\alpha }.
$$
and
$$
b_{ij}  = \indc{ | \Gamma_{\ell}^{ G \backslash S } (j) 
\cap I_0 | \ge \eta }.
$$
Define event
$$
\calA =  \left\{  |\Gamma_\ell^{G \backslash S} (j)| \ge   \frac{(np)^\ell }{2^{\ell-1} \log (np) }, \; \forall S \text{ s.t. } |S| \le 3,\;  \forall j \text{ s.t. } \frac{np}{\log (np) } \le |\Gamma_1^{G \backslash S} (j)| \le 4np \right\}.
$$
Note that $G\backslash S \sim \calG(n-|S|,p)$ and 
$(4np)^\ell =o(n)$. 
Applying \prettyref{cor:neighbor_grow_sparse} together with
union bounds, we
get that 
$$
\prob{\calA} \ge 1- n |S| 
\exp \left\{ -\Omega\left((np)^2/\log (np) \right) \right\} \ge 1-n^{\omega(1)}. 
$$ 
Define event $\calB$ such that for every vertex $i$, there is at most $1$ neighbor $j$ in $G$ such that 
$|\Gamma_1^{G \backslash S} (j) | \le (np)/\log (np)$.
Recall $\calE$ is the event that the maximum degree in $G$ is at most $4np$. In view of \prettyref{lmm:low_degree_children}, 
we have that $\prob{\calB \cap \calE|} \ge 1-n^{-1+o(1)}$. 

Recall that $\ell = \lfloor \frac{
(1-\epsilon)\log n}{\log (np) } \rfloor$ and $\eta = 4^{2\ell+2} n^{1-2\epsilon} \alpha$. Then for sufficiently large $n$, 
$$
\frac{(np)^\ell }{2^{\ell-1} \log (np) }
\ge  \frac{4 \eta}{\alpha} . 
$$
Hence, on event $\calA \cap \calB \cap \calE$,
$$
\sum_{j \in \Gamma_1^G(i) } a_{ij} \ge d_i -1, \quad \forall i. 
$$

Let 
$$
\calX = \cup_{i,j} \left( \left\{ a_{ij}=1   \right\}
\cap \left\{ b_{ij} =0 \right\} \right)
$$
Then on event $\calX^c$, for all $i,j$ such that $a_{ij}=1$, 
$b_{ij}=1$; thus $a_{ij} \le b_{ij}$ for all $i,j$. 
Hence, on event $\calA \cap \calB \cap \calE \cap \calX^c$,
we have
$$
\sum_{j \in \Gamma_1^G(i) } b_{ij} \ge d_i -1, \quad \forall i.
$$

It remains to show $\prob{ \calA \cap \calB \cap \calE \cap \calX^c} \ge 1-n^{-1+o(1)}$, which further reduces to proving  $\prob{\calX}\le n^{-1+o(1)}$ by the union bound. 
Note that 
\begin{align*}
\prob{a_{ij} =1, b_{ij} =0  }
& \le \prob{b_{ij} =0 \mid a_{ij} =1} \\
& \le \prob{ \Bin \left( \lfloor  4 \eta /\alpha \rfloor,  \alpha  \right)  \le \eta } \\
& \le e^{ - \eta },
\end{align*}
where the last inequality follows from the Binomial tail bound~\prettyref{eq:binomial_lower}. 
By the union bound, we have
\begin{align*}
\prob{\calX} \le \sum_{i,j}
\prob{ a_{ij} =1, b_{ij} =0 }
\le n^2 e^{ - \eta } \le n^{-2},
\end{align*}
where the last inequality holds due to the assumption that
$\eta \ge 4\log n$. 


\end{proof}

The second lemma is useful to conclude that 
the test statistic $Z_{u,v}$ given in \prettyref{eq:Z_sparse} is small for all distinct vertices $u,v$.

\begin{lemma}\label{lmm:noise_sparse}
Assume the same setup as \prettyref{lmm:signal_sparse}. 
With probability at least $1-4/n$, for all distinct $u,v$, there exists
a constant $C$ depending only on $\epsilon$ such that 
$$
\sum_{i \in \Gamma^G_1(u) } \sum_{j \in \Gamma^G_1(v) }
\indc{| N_\ell^{G \backslash\{u,v\}}   (i) \cap 
N_\ell^{G \backslash\{u,v\} }(j) \cap I_0| \ge \eta   
} \le C. 
$$
\end{lemma}
\begin{proof}
For two vertices $i,j$, define
$$
c_{ij} = \indc{ | N_{\ell}^{G} (i) \cap 
N_{\ell}^{G} (j)| \ge \eta/(4\alpha) }
$$
and an event
$$
\calC= \left\{ \max_i \sum_j c_{ij} \le 2n^{4\epsilon} \right\}
\cap \left\{ \max_j \sum_i c_{ij} \le 2n^{4\epsilon} \right\}
$$
In view of \prettyref{lmm:neighbor_intersect_small},
$\prob{\calC} \ge 1-2/n$.

Define  
$$
a_{ij} = \indc{ | N_{\ell}^{G} (i) \cap 
N_{\ell}^{G} (j) \cap I_0|   \ge \eta }.
$$
and an event 
$$
\calA= \left\{ \max_i \sum_j a_{ij} \le 2n^{4\epsilon} \right\}
\cap \left\{ \max_j \sum_i a_{ij} \le 2n^{4\epsilon} \right\}
$$

Moreover, let 
$$
\calY = \cup_{i,j} \left[ \left\{ c_{ij}=0   \right\}
\cap \left\{ a_{ij} =1 \right\} \right]
$$
Then on $\calY^c$, for all $i,j$ such that $c_{ij}=0$, 
$a_{ij}=0$; thus $a_{ij} \le c_{ij}$ for all $i,j$. 
Hence, $\calC \cap \calY^c \subset \calA$ and thus
$$
\prob{\calA} \ge \prob{ \calC \cap \calY^c}
\ge \prob{ \calC } - \prob{\calY }. 
$$
Note that 
\begin{align*}
\prob{c_{ij} =0, a_{ij} =1  }
& \le \prob{a_{ij} =1 \mid c_{ij} =0} \\
& \le \prob{ \Bin \left( \lfloor   \eta / 4\alpha \rfloor,  \alpha  \right)  \ge \eta } \\
& \le e^{ - 2\eta }
\end{align*}
By the union bound, we have
\begin{align*}
\prob{\calY} \le \sum_{i,j}
\prob{ c_{ij} =0, a_{ij} =1 }
\le n^2 e^{ - 2\eta } \le n^{-6},
\end{align*}
where the last inequality follows from the 
assumption that $\eta \ge 4 \log n$. 
Thus $\prob{\calA} \ge 1-3/n$.

Fix a pair of vertices $u \neq v$ in the sequel, and 
let 
$$
b_{ij} = \indc{| N_\ell^{G \backslash\{u,v\}}   (i) \cap 
N_\ell^{G \backslash\{u,v\} }(j) \cap I_0| \ge \eta   
}
$$
and
$$
\calB_{u,v} = \left\{ \max_i \sum_j b_{ij} \le 2n^{4\epsilon} \right\}
\cap \left\{ \max_j \sum_i b_{ij} \le 2n^{4\epsilon} \right\}
$$
Then by construction, 
$b_{ij} \le a_{ij}$ and thus
$\calA \subset \calB_{u,v}$.

Let $X_i = G(u,i)$ for $i \in [n]$ 
and $X_{n+j}=G (v, j)$ for $j \in [n]$. 
Define
$$
R_{u,v} = \sum_{i ,j \in [n] \backslash \{u, v\} } b_{ij} X_i X_{n+j},
$$
which is a degree-2 polynomial of $X_i$'s. 
Note that $\{b_{ij}; i,j \in [n] \backslash \{u, v\} \}$ only depends on $G\backslash \{u,v\}$
and hence is independent from
$X_i$'s. Moreover, $X_i$'s are i.i.d.\ $\Bern(p)$.

We condition on $\{b_{ij}\}$ such that event $\calB_{u,v}$ holds. 
Let  
$$
\mu_0= \expect{R_{u,v} \mid b } = p^2\sum_{ij} b_{ij} \le 2p^2n^{1+4\epsilon} \le 2 n^{-1+6\epsilon}.
$$
and
$$
\mu_1= 
\max\left\{ \max_i \expect{ \sum_j b_{ij} X_j \mid b},
\max_j \expect{ \sum_i b_{ij} X_i  \mid b} \right\}
\le 2p n^{4\epsilon} \le 2 n^{-1+5\epsilon}.
$$

By a concentration inequality for multivariate polynomials~\cite[Corollary 4.9]{Vu2002}, there exists a constant $C>0$ depending only on $\epsilon$ such that
\begin{align*}
\prob{  R_{u,v} \ge C \mid b}
\le  n^{-3} .
\end{align*}
Thus $\prob{ R_{u,v} \ge C \mid \calB_{u,v} } \le n^{-3}.$
Define event 
$
\calR_{u,v} =
\left\{ R_{u,v} \le C  \right\}
$
and 
$\calR=\cap_{u \neq v} \calR_{u,v}$.
It follows that  
$$
\prob{ \calR_{u,v} \cap \calB_{u,v}} 
\le \prob{ \calR_{u,v} \mid \calB_{u,v} }
\le   n^{-3} . 
$$
Since $\calA \subset \calB_{u,v}$, it 
further follows that 
$
\prob{ \calR_{u,v} \cap \calA }
\le  n^{-3} . 
$
By a union bound over $u,v$, we have
$
\prob{ \calR \cap \calA }
\le  n^{-1}. 
$
Hence, 
$
\prob{ \calR }
\le \prob{\calR \cap \calA} + \prob{ \calA^c }
\le  4/n. 
$

\end{proof}

With \prettyref{lmm:signal_sparse} and \prettyref{lmm:noise_sparse}, 
we are ready to finish the proof of
\prettyref{thm:sparse_improved}.

\begin{proof}[Proof of \prettyref{thm:sparse_improved}]

Recall that $\tau$ is given in \prettyref{eq:def_tau} and the definition of high-degree vertices. 
We first prove  that \prettyref{alg:graph_matching_ell_hop_witness} correctly matches the high-degree
vertices in $G_1^* \wedge G_2$ with high probability. 

Recall the definition of $Z$ give in \prettyref{eq:Z_sparse}.
Applying \prettyref{lmm:signal_sparse} with $G= G_1^*\wedge G_2$, we get that with high probability, 
for all high-degree vertices $u$, 
$$
Z_{u,u} \ge \tau -1 = \frac{nps^2}{\log (nps^2) } -1 . 
$$

Moreover, by definition, 
\begin{align*}
w^{u,v}_{i,j} 
& \le \left| \left\{ k \in I_0: \pi_0(k) \in N^{G_{1} \backslash \{ u, v\} }_{\ell} (i), \;
k \in N^{G_{2} \backslash \{u,v\} }_{\ell} (j) \right\} \right| \\
 & \le
 \left| N^{G_1^* \vee G_2 \backslash \{ u, v\} }_{\ell} (i)
 \cap  N^{G_{1}^* \vee G_2 \backslash \{u,v\} }_{\ell} (j) \cap I_0   \right|.
\end{align*}
Applying \prettyref{lmm:noise_sparse} with $G= G_1^*\vee G_2$, we get that with high probability, 
$$
Z_{u,v} \le C, \quad \forall u \neq v
$$
for a constant $C>0$ only depending on $\epsilon$. 
Since for sufficiently large $n$, $\tau \ge C+1$, it follows
that \prettyref{alg:graph_matching_ell_hop_witness} correctly
matches all high-degree vertices with high probability.

The proof of correctness for matching low-degree vertices is the same as \prettyref{alg:graph_matching_ell_hop}
and thus omitted. 

\end{proof}

\begin{appendices}

\section{Proof of \prettyref{thm:converse}}\label{app:converse}
\begin{proof}[Proof of \prettyref{thm:converse}]
Suppose $nps^2-\log n=c$ for $c<+\infty$. 
Since $G_1^* \wedge G_2 \sim \calG(n, ps^2)$, 
classical random graph theory shows that 
the distribution of the number of isolated vertices in $G_1^* \wedge G_2$ converges to $\Pois(e^{-c})$, see, \eg, \cite[Theorem 3.1]{bollobas1998random}.
Let $\calF_1$ denote the event that there are at least two isolated vertices in $G_1^* \wedge G_2$. 
Then $\prob{\calF_1} =\Omega(1)$. 

Let $\calF_2$ denote the event that there are at least two isolated vertices that are unseeded in $G_1^* \wedge G_2$. Since each vertex is seeded  with probability $\alpha$ independently across different vertices and from the graphs $G_1$ and $G_2$, it follows that 
$\prob{F_2} \ge \prob{F_1} (1-\alpha)^2=\Omega
\left((1-\alpha)^2 \right) $. 

Since the prior distribution of $\pi^*$ is uniform, the maximum likelihood estimator $\hat{\pi}_{\rm ML}$  minimizes the error probability $\prob{ \hat{\pi} \neq \pi^*}$ among all possible estimators and thus we only need to find when MLE fails. 

Recall that $I_0$ is the seed set. 
Let $\calS$ denote the set of all possible permutations $\pi$ such that $\pi(i)=\pi^*(i)$ for $i \in I_0$. 
Under the seeded model $\calG(n,p;s,\alpha)$, the maximum likelihood estimator $\hat{\pi}_{\rm ML}$ is given by the minimizer of the (restricted) quadratic assignment problem, namely, 
$$
\hat{\pi}_{\rm ML} \in \arg \min \min_{ \pi \in \calS }\|  G_1  -  \Pi G_2 \Pi^\top  \|_F^2,
$$
where $\Pi$ is the permutation matrix corresponding to permutation $\pi$; or equivalently, 
$$
\hat{\pi}_{\rm ML} \in \arg \max _{\pi \in \calS} \iprod{G_1}{\Pi G_2 \Pi^\top }.
$$

Let $I$ denote 
the union of the initial seed set and the set of all
 non-isolated vertices in $G_1^* \wedge G_2$. Then $I^c$
 is the set of isolated vertices that are unseeded in $G_1^* \wedge G_2$.
Let  $\tilde{\calS}$ denote the set of 
all possible permutations $\pi$ such that $\pi(i)=\pi^*(i)$ for $i \in I$. Then $\pi^* \in \tilde{\calS} \subset \calS$. Note that for any $\pi \in \tilde{\calS}$, we have
\begin{align*}
\iprod{ G_1  }{\Pi G_2 \Pi^\top }
& \ge \sum_{(i, j) \in I\times I} G_1( \pi(i),\pi(j) ) G_2(i, j)  \\
& \overset{(a)}{=} \sum_{(i, j) \in I\times I} G_1( \pi^*(i), \pi^*(j) ) G_2(i,j) \\
& =\sum_{(i, j) } G_1( \pi^*(i), \pi^*(j) ) G_2(i,j),
\end{align*}
where $(a)$ follows from $\pi(i)=\pi^*(i)$ for $i \in I$;
the last equality holds due to $G_1( \pi^*(i), \pi^*(j) ) G_2(i,j) =0$ for all $(i,j) \notin I \times I$. 
Hence, there at at least $|I^c|! -1 $ different permutations in $\tilde{\calS}$ whose likelihood is at least as large as the ground truth $\pi^*$, and hence the 
MLE is correct with probability at most $1/(|I^c|!-1).$ Note that on event $\calF_2$, $|I^c| \ge 2$; hence,  
MLE is correct with probability at most $1/2.$ 
In conclusion, MLE is correct with probability at most $(1/2)\prob{\calF_2} =\Omega((1-\alpha)^2)$.

%

\end{proof}

\section{Proof of \prettyref{lmm:graph_properties} }

\begin{proof}
Claim (i): For each vertex $i$, its degree $d_i \sim \Binom(n-1,p)$. 
By the union bound, the probability that $G$ has an isolated vertex is 
$$
n(1-p)^{n-1} \le  n e^{ - (n-1) p} =o(1),
$$
where the last equality holds due to the assumption that  $np - \log n \to +\infty$. \\

Claim (ii):
Fix any pair of two distinct vertices $i,j$, define
\begin{align*}
\calE_{ij} = \{ G(i,j) =1 \} \cap \left\{ d_i \le \tau \right\} \cap 
\left\{  d_j \le \tau \right\}.
\end{align*}
It suffices to show
$$
\prob{ \cup_{i\neq j} \calE_{ij}} \le n^{-1+o(1)}.
$$
Note that 
\begin{align*}
\prob{  d (i) \le \tau, d (j) \le \tau | G(i,j)=1}
& = \left( \prob{ \Bin ( n-2, p ) \le \tau -1} \right)^2 \\
& \le \left( \prob{ \Bin ( n-2, p ) \le \tau }  \right)^2
\end{align*}
In view of Binomial tail bounds given in \prettyref{thm:binom_lower_tail} and $\tau=o(np)$, we have that
\begin{align*}
\prob{ \Bin ( n-2, p ) \le \tau  } 
\le \exp \left( - (n-2) p \left( 1  - \sqrt{ \frac{\tau }{ (n-2) p}} \right)^2 \right)
=\exp \left( - (1-o(1)) np  \right).
\end{align*}
Combining the last two displayed equations yields that
\begin{align*}
\prob{ \calE_{ij}} 
& = \prob{G(i,j)=1} \prob{  d (i) \le \tau, d (j) \le \tau | G(i,j)=1} \\
& \le  p \exp \left( - 2 (1-o(1)) np \right)
\end{align*}
By the union bound,
\begin{align*}
\prob{ \cup_{i\neq j} E_{ij} } & \le n^2 \prob{E_{ij}} \\
& \le n^2 p \exp \left( - 2 (1-o(1)) np \right)
=n^{-1+o(1)},
\end{align*}
where the last equality holds due to $np-\log n\to +\infty$.
\end{proof}

\section{Proof of \prettyref{lmm:low_degree_path} }
\begin{proof}
Let $d_i$ denote the degree of vertex $i$ in $G_1^* \wedge G_2$ and 
$A$ denote the adjacency matrix of $G_1^* \vee G_2.$ 
For every pair of three distinct vertices $i,j,k$, define
$$
\calF_{ijk} = \{ A_{ik}=1, A_{jk}=1 \}  \cap \left\{ d_i  \le \tau \right\} \cap \left\{ d_j  \le \tau \right\}.
$$
It suffices to show that $\prob{\cup_{i,j,k} F_{ij}} \le n^{-1+o(1)}$. 
Since $G_1^*\vee G_2 \sim \calG(n, ps(2-s))$, it follows that
$$
\prob{ A_{ik}=1, A_{jk}=1 } = \prob{A_{ik}=1} \prob{A_{jk}=1} = \left( ps (2-s) \right)^2  
\le p^2. 
$$
Moreover, since $G_1^* \wedge G_2 \sim \calG(n, ps^2)$, it follows that 
\begin{align*}
  \prob{  \left\{ d_i  \le \tau \right\} \cap \left\{ d_j  \le \tau \right\} 
 \mid A_{ik}=1, A_{jk}=1}
  \le \left( \prob{ \Binom(n-3, ps^2) \le \tau  } \right)^2.
 \end{align*}
 In view of Binomial tail bound \prettyref{eq:binomial_lower} and $\tau=o(nps^2)$, we have that
 \begin{align*}
 \prob{ \Binom(n-3, ps^2) \le \tau } 
 & \le \exp \left( - (n-3) ps^2 \left( 1  - \sqrt{ \frac{\tau}{ (n-3) ps^2}} \right)^2 \right) \\
 & =\exp \left( - n ps^2 \left( 1 -o(1 ) \right) \right)
 \end{align*}
 It follows that
 $$
 \prob{ \calF_{ijk}} \le p^2   \exp \left( - 2 n ps^2   \left( 1 -o(1 ) \right) \right)
 $$
By the union bound, we have that
$$
\prob{\cup_{i,j,k} F_{ij}} \le n^3 p^2 \exp \left( - 2 n ps^2   \left( 1 -o(1 ) \right) \right) =n^{-1+o(1)}.
$$
where the last equality holds due to $nps^2 \ge \log n$ 
and $\log (np) = o (nps^2).$ 
\end{proof}

\section{Neighborhood Exploration in $\calG(n,p)$}\label{app:neighborhood_expansion}
Throughout this section, we assume graph $G \sim \calG(n,p)$
with $n p \ge \log n$. We first claim that the max degree in $G$ 
is at most $4np$ with probability at least $1-1/n$.  
\begin{lemma}\label{lmm:max_degree}
Assume graph $G \sim \calG(n,p)$ with $np \ge \log n$.  Let 
\begin{align}
\calE=\left\{ \max_{v \in V(G)} d_v \le  4np \right\}.
\label{eq:def_calE}
\end{align}
Then 
$$
\prob{\calE} \ge 1- n^{-1}.  
$$
\end{lemma}
\begin{proof}
By the Binomial tail bound~\prettyref{eq:binomial_upper},
$$
\prob{ d_i \ge 4 n p  } =  \prob{ \Binom(n-1, p) \ge 4 n p } \le \exp( - 2 np).  
$$
The proof follows by the union bound and the assumption
that $np \ge \log n$. 
\end{proof}

We fix a vertex $u$ throughout this section, and abbreviate $\Gamma^G_k(u)$ as $\Gamma_k(u)$ and $N^G_k(u)$ as $N_k(u)$ for
simplicity. We are interested in studying the growth of $|\Gamma_k(u)|$ as $k$ increases. Note that $|\Gamma_1(u)|$ is the degree $d_u$ of vertex $u$ in $G$.
Since the average
degree is $(n-1)p$, we expect typically $|\Gamma_k(u)|$ grows as $(np)^{k}.$ This is indeed true in the dense regime with $np \ge n^{\epsilon}$.

\subsection{Dense Regime}
The following lemma is adapted from~\cite[Lemma 10.9]{bollobas1998random}.

\begin{lemma}\label{lmm:bollobas_neighbor_expansion}
Suppose $np \ge n^{\epsilon}$ for an arbitrarily small constant $\epsilon>0$ and $d$ is chosen such that 
$$
(np)^{d-1} \le \frac{n}{8} \quad \text{ and } \quad (np)^d \ge n \log n
$$
If $n$ is sufficiently large, then with probability at least $1-n^{-10}$, the following claims hold:
\begin{enumerate}[(i)]
\item For every vertex $u$, 
$$
\left|  \Gamma_{k} (u) - (np)^{k} \right| \le \frac{1}{4} (np)^{k}. \quad \forall 0 \le k \le d-1.   
$$
\item For every two distinct vertices $u$ and $v$, 
\begin{align*}
\left| N_{d-1} (u) \cap N_{d-1} (v) \right| & \le 8 n^{2d-3} p^{2d-2}.
\end{align*}
\end{enumerate}
\end{lemma}

\prettyref{lmm:bollobas_neighbor_expansion} also upper bounds 
$|\Gamma_{d-1}(u) \cap \Gamma_{d-1}(v)|$ for two distinct vertices $u, v$ by $8p^{2d-2}n^{2d-3}$. To see this intuitively, note that in the dense regime, $\Gamma_{d-2}(u) \cap \Gamma_{d-2}(v)$ is typically 
of a much smaller size comparing to either $\Gamma_{d-2}(u) $ or $\Gamma_{d-2}(v)$. Hence, the majority of vertices $w$ 
 in $\Gamma_{d-1}(u) \cap \Gamma_{d-1}(v)$ are connected to some vertex in $\Gamma_{d-2}(u) \setminus \Gamma_{d-2}(v)$ and to some vertex in $\Gamma_{d-2}(v) \setminus \Gamma_{d-2}(u)$.  For a given vertex $w \notin N_{d-2}(u) \cup N_{d-2}(v)$, since $|\Gamma_{d-2}(u) \setminus \Gamma_{d-2}(v)| \le |\Gamma_{d-2}(u)| \le 2 (np)^{d-2}$ and similarly for
 $|\Gamma_{d-2}(v) \setminus \Gamma_{d-2}(u)|$, $w$ connects to some vertex in $\Gamma_{d-2}(u) \setminus \Gamma_{d-2}(v)$ with probability at most $2p (np)^{d-2}$, and connects to some vertex in $\Gamma_{d-2}(v) \setminus \Gamma_{d-2}(u)$ independently with probability $2p(np)^{d-2}$. Moreover, there are at most $n$ such potential vertices $w$ to consider. Hence, we expect $|\Gamma_{d-1}(u) \cap \Gamma_{d-1} (v)|$ to be smaller than $ 2n [ 2p (np)^{d-2} ]^2= 8p^{2d-2}n^{2d-3}.$

\subsection{Sparse Regime}
In contrast, in the sparse regime where
$$
np - \log n \to + \infty.
$$
there exist vertices with small degrees, \ie, $|\Gamma_1(u)|$ is much smaller than
$np.$ 
Hence, we cannot expect $|\Gamma_k(u)|$ grows like $(np)^k$ for all vertices $u$. Nevertheless, the following lemma shows that
conditional on $|\Gamma_1(u)|$ is large, then $|\Gamma_k (u)| \asymp (np) |\Gamma_{k-1}(u)| $
for all $2 \le k \le d$ for some $d$
with high probability.

\begin{lemma}\label{lmm:neighbor_grow_sparse}
Suppose 
\begin{align}
np \ge \log n \quad \text{ and } \quad  
 p (4np)^{d-1} = o(1).  \label{eq:neighbor_p_assumption}
\end{align} 
Let $u$ be a fixed vertex. For each $1 \le k \le d$, define
 $$
 \calQ_k =\left\{ |\Gamma_{k}(u)| \in I_k = \left[ \tau \left( \frac{np}{2} \right)^{k-1}, (4np)^{k} \right] \right\}
 $$ 
 for $1 \le \tau \le np$. Then  for $2 \le k \le d$, 
$$
\prob{\calQ_k \mid \calQ_1, \ldots, \calQ_{k-1}} \ge 
1 -\exp \left( -  \Omega \left(\tau \left( \frac{np}{2} \right)^{k-1} \right) \right).
$$
It readily follows that
$$
\prob{\calQ_d \cap \calQ_{d-1} \cap \cdots \cap \calQ_2 \mid \calQ_1} 
\ge  1 - \exp \left( -  \Omega \left(\tau np \right) \right).
$$
\end{lemma}
\begin{proof}
Fix $2 \le k \le d$. 
Conditional on $\Gamma_{k-1}(u)$ and $N_{k-1}(u)$, the probability of a given
vertex $v \notin N_{k-1}(u)$ being connected to some vertices in $\Gamma_{k-1}(u)$ is
$$
p_k \triangleq 1- (1-p)^{|\Gamma_{k-1}(u)|}.
$$
Therefore, conditional on $|\Gamma_{k-1}(u)|$ and $|N_{k-1}(u)| $, 
$$
|\Gamma_{k}(u)| \sim  \Bin \left( n - |N_{k-1}(u)|, p_k  \right)
$$
Note that conditional on $\calQ_1, \ldots, \calQ_{k-1}$, 
$$
|N_{k-1}(u)| = \sum_{i=0}^{k-1} | \Gamma_i (u) | \le \sum_{i=0}^{k-1} (4np)^i = \frac{ (4 np)^k -1 }{4np-1} =o(n), 
$$
where the last equality holds due to the assumption \prettyref{eq:neighbor_p_assumption} and 
$k \le d.$
Moreover, in view of the assumption \prettyref{eq:neighbor_p_assumption}, conditional on $\calQ_1, \ldots, \calQ_{k-1}$, 
$$
 \left( 1-o(1) \right) p  \tau \left( \frac{np}{2} \right)^{k-2}  \le p_k \le  p (4np)^{k-1}.
$$
Hence, for $2 \le k \le d$,
\begin{align*}
\prob{   |\Gamma_{k}(u)|  \notin I_k \mid \calQ_1, \ldots, \calQ_k }
& \le \prob{ \Bin \left( n -o(n), \left( 1-o(1) \right) p  \tau \left( \frac{np}{2} \right)^{k-2} \right) \le \tau \left( \frac{np}{2} \right)^{k-1} }  \\
& + \prob{ \Bin \left( n , p (4np)^{k-1}\right) \ge  (4np)^{k}  } \\
& \le \exp \left( - \Omega \left(\tau \left( \frac{np}{2} \right)^{k-1} \right) \right) + \exp \left( - 4^{k-1} (np)^k \right) \\
& \le  \exp \left( -  \Omega \left(\tau \left( \frac{np}{2} \right)^{k-1} \right) \right).
\end{align*}
Finally, we note that 
\begin{align*}
\prob{\calQ_d \cap \calQ_{d-1} \cap \cdots \cap \calQ_2 \mid \calQ_1} 
& =\prob{\calQ_2 \mid \calQ_1}  \prob{ \calQ_3 \mid \calQ_1, \calQ_2}  \cdots  \prob{\calQ_d \mid \calQ_1, \ldots, \calQ_{d-1} } \\ 
& \ge \prod_{i=0}^{d-1} \left( 1 - \exp \left( -  \Omega \left(\tau \left( \frac{np}{2} \right)^{k-1} \right) \right) \right) \\
& \ge 1 - \sum_{i=0}^{d-1}  \exp \left( -  \Omega \left(\tau \left( \frac{np}{2} \right)^{k-1} \right) \right) \\
& \ge 1- \exp \left( -  \Omega \left(\tau np \right) \right).
\end{align*}
\end{proof}

With \prettyref{lmm:neighbor_grow_sparse}, we have the following 
immediate corollary. 
\begin{corollary}\label{cor:neighbor_grow_sparse}
Suppose \prettyref{eq:neighbor_p_assumption} holds. Define
event 
$$
\calQ = \left\{  |\Gamma_{k } (u) | \in I_k, \; \forall 1 \le k \le d,  \;  \forall u \text{ s.t. }
\tau \le |\Gamma_1(u)| \le 4np \right\}
$$
Then 
$$
\prob{\calQ} \ge 1- n \exp \left( -  \Omega \left(\tau np \right) \right).
$$
\end{corollary}
\begin{proof}
Note that 
$$
\calQ^c = \cup_{u} \left( \left\{ \tau \le |\Gamma_1(u)| \le 4np  \right\}  \cap \left\{  |\Gamma_{k } (u) | \notin I_k, \; \forall 1 \le k \le d \right\} \right).
$$
Hence, it follows from the union bound that 
\begin{align*}
\prob{\calQ^c } &\le \sum_u 
\prob{ \left\{ \tau \le |\Gamma_1(u)| \le 4np  \right\}  \cap \left\{  |\Gamma_{k } (u) | \notin I_k, \; \forall 1 \le k \le d \right\}  } \\
& \le \sum_u \prob{ |\Gamma_{k } (u) | \notin I_k, \; \forall 1 \le k \le d  \mid   \tau \le |\Gamma_1(u)| \le 4np } \\
& \le n \exp \left( -  \Omega \left(\tau np \right) \right),
\end{align*}
where the last inequality follows from \prettyref{lmm:neighbor_grow_sparse}.
\end{proof}


Next, we upper bounds
$|N_{d}(u) \cap N_{d}(v)|$ for two distinct vertices $u, v$ in the sparse regime. We need to introduce 
 $$
 \Gamma^*_{k,\ell} (u, v) = \left\{  
 w \in \Gamma_k(u) \cap \Gamma_\ell(v) : 
 \Gamma_1 ( w) \cap \left( \Gamma_{k-1}(u) \setminus \Gamma_{\ell-1}(v) \right) \neq \emptyset, \; \Gamma_1 ( w) \cap \left( \Gamma_{\ell-1}(v) \setminus \Gamma_{k-1}(u) \right) \neq \emptyset
 \right\}
 $$
 and we abbreviate $\Gamma^*_{k,k}(u,v)$ as $\Gamma^*_k (u,v)$
 for simplicity. 
By definition, for any $d \ge 1$, 
$$
\Gamma_{d} (u) \cap \Gamma_{d} (v) \subset \cup_{k=1}^{d} \Gamma_{d-k}\left( \Gamma^*_{k} (u, v) \right). 
$$
and 
$$
N_{d}(u) \cap N_d(v) \subset
\cup_{\ell=-d}^d \cup_{k=0}^d N_{d-k - \max\{\ell, 0\} } \left( \Gamma^*_{k+\ell, k} (u,v) \right).
$$
The following lemma gives an upper bound to $\left| \Gamma^*_{k,\ell} (u,v)  \right|$ in high probability.

\begin{lemma}\label{lmm:neighbor_intersect}
For two distinct vertices $u, v$, define
$$
\Delta_{k,\ell} = \left\{ | \Gamma_{k-1}(u) |  \le (4np)^{k-1}, \; 
| \Gamma_{\ell-1}(v) |  \le (4np)^{\ell-1}\right\}. 
$$
For all $k \ge 1$,  
\begin{align}
\prob{ \left| \Gamma^*_{k,\ell}(u,v) \right| \ge \gamma_{k+\ell} \mid \Delta_{k,\ell} } \le \frac{1}{n^8}, \label{eq:neighbor_interset_high_prob}
\end{align}
where 
\begin{align}
\gamma_{k}  = 
\begin{cases}
24 \log n & \text { if } np^2 ( 4np)^{k-2} \le 4 \log n \\
4 np^2 ( 4np)^{k-2} & \text { o.w. }
\end{cases}
\end{align}
\end{lemma}

\begin{proof}
Conditional on $\calN_{k-1}(u), \Gamma_{k-1}(u)$ and $\calN_{\ell-1}(v),\Gamma_{\ell-1}(v)$, the probability that a  vertex $w \notin \calN_{k-1}(u) \cup \calN_{\ell-1}(v)$ being connected to some vertex in $\Gamma_{k-1}(u) \setminus \Gamma_{\ell-1}(v)$ is 
$$
1- (1-p)^{| \Gamma_{k-1}(u) \setminus \Gamma_{\ell-1}(v)| } \le p | \Gamma_{k-1}(u) \setminus \Gamma_{\ell-1}(v)| \le p | \Gamma_{k-1}(u)|  .
$$
Similarly, the probability that $w$ is connected to some vertex in $\Gamma_{\ell-1}(v) \setminus \Gamma_{k-1}(u)$ is 
$$
1- (1-p)^{| \Gamma_{\ell-1}(v) \setminus \Gamma_{k-1}(u)| } \le p | \Gamma_{\ell-1}(v) \setminus \Gamma_{k-1}(u)| \le p | \Gamma_{\ell-1}(v)|.
$$
Since $\Gamma_{k-1}(u) \setminus \Gamma_{\ell-1}(v)$ is disjoint from $\Gamma_{\ell-1}(v) \setminus \Gamma_{k-1}(u)$, the probability that $w \in \Gamma^*_{u,v}$ is at most $p^2 | \Gamma_{k-1}(u)| | \Gamma_{\ell-1}(v)|$. 
Moreover, there are at most $n$ vertices $w \notin \calN_{k-1}(u) \cup \calN_{\ell-1}(v)$.  
Hence,
$$
\prob{ \left| \Gamma^*_{k,\ell}(u,v) \right| \ge \gamma_{k+\ell} \mid \Delta_{k,\ell} } \le 
\prob{ \Bin \left( n, p^2 (4np)^{k+\ell-2 } \right)  \ge \gamma_{k+\ell} }.
$$
If $n p^2 (4np)^{k+\ell-2 } \le 4 \log n$, then by the choice of $\gamma_{k+\ell}=24 \log n$, we have $\gamma_{k+\ell} \ge 6 n p^2 (4np)^{k+\ell-2 }$. It follows from \prettyref{eq:binomial_upper_2} that 
$$
\prob{ \Bin \left( n, p^2 (4np)^{k+\ell-2 } \right)  \ge \gamma_{k+\ell} }
\le 2^{-\gamma_{k+\ell} } =2^{-24 \log n} \le \frac{1}{n^8}.
$$
If $n p^2 (4np)^{k+\ell-2 } \ge 4 \log n$, then by the choice of 
$\gamma_{k+\ell} = 4 n p^2 (4np)^{k+\ell-2 }$, it follows from \prettyref{eq:binomial_upper} that 
$$
\prob{ \Bin \left( n, p^2 (4np)^{k+\ell-2 } \right)  \ge \gamma_{k+\ell} }
\le  \exp \left( - 2n p^2 (4np)^{k+\ell-2 } \right) \le \frac{1}{n^8}.
$$
\end{proof}

With \prettyref{lmm:neighbor_intersect}, we are ready to upper bound 
$|N_{d} (u) \cap N_{d} (v) |$ for $d$ large enough. 

\begin{lemma}\label{lmm:neighbor_intersect_small}
For a given small constant $\epsilon>0$, choose an integer $1 \le d \le n$ such that
$$
\left( 4np \right)^d \ge n^{1-\epsilon}
$$
For each vertex $u$, define event
$$
\calR_u = \left\{ 
\sum_{v} \indc{ |N_{d} (u) \cap N_{d} (v) | > 4^{2d+1} p^{2d} n^{2d-1} } \le 2 n^{4\epsilon} \right\}
$$
and $\calR=\cap_u \calR_u$. Then  
 \begin{align}
 \prob{ \calR } \ge 1- 2n^{-1}.
 \end{align}
\end{lemma}
\begin{proof}
Define an event 
$$
\calA = \cap_{u \neq v} \cap_{1\le k \le d } \cap_{1\le k \le \ell}  \left\{ |\Gamma^*_{k,\ell} (u, v) | \le \gamma_{k+\ell} \right\} 
$$
Recall $\calE$ defined in \prettyref{eq:def_calE}. Note that 
$$ 
(\calA \cap \calE)^c = (\calA^c \cap \calE) \cup \calE^c.
$$
Therefore,  
$$
\prob{ (\calA \cap \calE)^c } 
\le \prob{ \calA^c \cap \calE} + \prob{\calE^c}.
$$
Note that $\prob{\calE^c} \le 1/n$. Moreover,
\begin{align*}
\prob{ \calA^c \cap \calE} &\le \sum_{u\neq v} \sum_{ 1 \le k \le d } \sum_{1\le \ell \le d }
\prob{  \left\{  |\Gamma^*_{k,\ell} (u, v)| \ge \gamma_{k+\ell} \right\} \cap \calE} \\
& \overset{(a)}{\le} \sum_{u\neq v} \sum_{1 \le k \le d } \sum_{1\le \ell \le d }
\prob{  \left\{  |\Gamma^*_{k,\ell} (u, v)| \ge \gamma_{k+\ell} \right\} \cap \Delta_{k,\ell}  } \\
&\le \sum_{u\neq v} \sum_{ 1 \le k \le d } \sum_{1\le \ell \le d }
\prob{   |\Gamma^*_{k,\ell} (u, v)| \ge \gamma_{k+\ell}  \mid \Delta_{k,\ell} }
 \le  n^{-4},
\end{align*}
where $(a)$ follows from  $\calE \subset \Delta_k$
and the last inequality holds in view of 
\prettyref{lmm:neighbor_intersect} and $d \le n$. 
Therefore, $\prob{(\calA \cap \calE)^c} \le 2/n$. 

To prove the lemma, it suffices to argue that 
$ \calA \cap \calE \subset \calR$. 
To see this, let us assume that $\calA \cap \calE$ holds in the 
sequel. 
Note that 
$$
N_{d}(u) \cap N_d(v) \subset
\cup_{\ell=-d}^d \cup_{k=0}^d N_{d-k - \max\{\ell, 0\} } \left( \Gamma^*_{k+\ell, k} (u,v) \right).
$$
It follows that 
$$
\left| N_{d}(u) \cap N_d(v) \right| 
\le \sum_{\ell=-d}^d \sum_{k=0}^d 
\left| \Gamma^*_{k+\ell, k} (u,v) \right| (4np)^{d-k-\max\{\ell,0 \} }
$$
Set $k_0$
$$
k_0= \left \lfloor \frac{2\epsilon \log n} {\log (4np) } \right \rfloor
$$
Then 
$$
|N_{2k_0}(u)| \le \sum_{k=0}^{2k_0} (4np)^k = \frac{(4np)^{2k_0+1} -1 }{4np -1 } \le 2 (4np)^{2k_0} \le 2n^{4\epsilon}, 
$$
where the second-to-the-last inequality holds due to $2np \ge 1$. 
Note that for all $v \notin N_{2k_0}(u)$, we have
$$
| \Gamma^*_{k,\ell} (u,v) | =0 , \quad \forall 0 \le k +\ell \le 2k_0
$$
and thus
\begin{align*}
\left| N_{d} (u) \cap N_{d} (v) \right| 
& \le \sum_{\ell=-d}^d \sum_{k=0}^d \indc{0 \le k+\ell \le d} \indc{2k+\ell \ge 2k_0+1} \gamma_{2k+\ell} (4np)^{d-k-\max\{\ell,0\} } \\
& \le \sum_{\ell=-d}^d \sum_{k=0}^d \indc{0 \le k+\ell \le d} \indc{2k+\ell \ge 2k_0+1}   \left( 24 \log n + 4 np^2 
(4np)^{2k+\ell-2} \right)  (4np)^{ d -k -\max\{\ell,0\}  } \\
& \le  192 \log n (4np)^{ d-k_0-1/2 } 
+ 32 np^2 (4np)^{ 2d-2 }  \\
& \le 64 np^2 (4np)^{ 2d-2 }  = 4^{2d+1} p^{2d} n^{2d-1},
\end{align*}
where the last inequality holds due to 
$(4np)^{d+k_0+1/2} \ge 6 n \log n$ for $n$ sufficiently large. 
Hence, for every $u$, 
$$
\sum_{v} \indc{ |N_{d} (u) \cap N_{d} (v) | > (4)^{2d+1} p^{2d} n^{2d-1} }
\le |N_{2k_0}(u)|
\le 2 n^{4\epsilon}. 
$$
As a consequence, $ \calA \cap \calE \subset \calR$. 
\end{proof}

\subsection{Graph Branching in Sparse Regime}
\label{app:graph_branching_process}
In this subsection, we describe a branching process to explore the vertices in $N_k(u)$. See, \eg, \cite[Section 11.5]{AlonSpencer08} for a reference.

\begin{definition}[Graph Branching Process]
We begin with $u$ and apply breadth-first-search to explore the vertices in $N_k(u)$. In this process,
all vertices will be ``live'', ``dead'', or ``neutral''. The live vertices will be contained in a queue. Initially,
at time $0$, $u$ is live and the queue consists of only $u$, and all the other vertices are neutral. At each time step
$t$,  a live vertex $v$ is popped from the head of the queue, and we check all pairs $\{v, w\}$ for all neutral vertices
$w$ for adjacency. The poped vertex $v$ is now dead and those neutral vertices $w$ adjacent to $v$ are added to the end
of the queue (in an arbitrary order) and now are live. The process ends when the queue is empty.
\end{definition}

Note that such a branching process constructs a tree $T(u)$ rooted at $u$. In particular, 
at each time step, those neutral vertices $w$ adjacent to the poped vertex $v$ can be viewed as children of $v.$ 
For each vertex $v$ in $T(u)$, 
abusing notation slightly, we let 
$T_k(v) $ denote the subtree rooted at $v$ of depth $k$  in $T(u)$ 
and $\Pi_k(v)$ denote the 
set of vertices at distance $k$ from root $v$ in subtree $T_k(v)$.
Note that by construction, $\Pi_k(u)=\Gamma_k(u)$ for root $u$.

We are interested in bounding $|\Pi_k(v)|$ for each children $v$ of root $u.$ The following lemma shows that with high probability, 
for all childen $v$ of root $u$ such that 
$|\Pi_{1}(v)| \ge \tau$, $|\Pi_{k}(v)|$ grows at least  as $\tau \left( np/2 \right)^{k-1}$.
\begin{lemma}\label{lmm:branching_process}
Let $u$ be the root vertex and $1 \le \tau \le np$. Define 
$$
\calF_1 = \left\{ |\Pi_1(u)| \le 4np \right\} \cap \left\{ |\Pi_{1}(v)| \le 4np ,  \forall v \in \Pi_1(u) \right\},
$$
and for each $2 \le k \le d$ define
 $$
 \calF_k =
 \left\{  |\Pi_{k}(v)| \le (4np)^{k},  \forall v \in \Pi_1(u)   \right\}
 \cap
 \left\{ |\Pi_{k}(v)| \ge \tau \left( np/2 \right)^{k-1}, 
 \forall v \in \Pi_1(u) \text{ s.t. }  |\Pi_{1}(v)| \ge \tau \right\}
 $$ 
Suppose 
\begin{align}
np \ge \log n \quad \text{ and } \quad  
  (4np)^{d+1} = o(n).  \label{eq:neighbor_p_assumption_2}
\end{align} 
Then  for $2 \le k \le d$, 
$$
\prob{\calF_k \mid \calF_1, \ldots, \calF_{k-1}} \ge 
1- 
8np \exp \left(  - \Omega \left( \tau \left( \frac{np}{2} \right)^{k-1} \right) \right).
$$
It readily follows that
$$
\prob{\calF_d \cap \calF_{d-1} \cap \cdots \cap \calF_2  \mid \calF_1 } 
\ge  1  - 8np \exp \left( -\Omega\left(\tau np \right) \right). 
$$
Moreover, by letting 
$$ 
\calA_u = \left( \calF_d \cap \calF_{d-1} \cap \cdots \calF_2 \right) \cup \calF_1^c,
$$
we have 
$$
\prob{\calA_u^c} \le 8np \exp \left( -\Omega\left(\tau np \right) \right).
$$
\end{lemma}
\begin{proof}
Fix $2 \le k \le d.$
Suppose the neighbors of root vertex $u$ are added to the queue in the order of 
$v_1, v_2, \ldots, v_{d_u}$, where $d_u= | \Pi_1(u)| $. 
Then by the branching process aforementioned, 
$\Pi_{k}(v_1), \ldots, \Pi_{k}(v_{i-1})$ are revealed
before $\Pi_{k}(v_{i})$. 

Fix $1 \le i \le d_u$ and define 
$$
\calF_{k,i} = 
\left\{ | \Pi_{k} (v_j) | \le (3np)^k, \; \forall j \in [i] \right\}
\cap \left\{ |\Pi_{k}(v_j)| \ge \tau \left( \frac{np}{2} \right)^{k-1}, 
 \forall j \in [i] \text{ s.t. }  |\Pi_{1}(v_j)| \ge \tau \right\}.
$$
Then $\calF_k=\calF_{k,d_u}.$

Conditional on $\Pi_{k-1}(v_i)$, the probability of a given
neutral vertex $w$ being connected to some vertices in $\Pi_{k-1}(v_i)$ is
$$
p_k \triangleq 1- (1-p)^{|\Pi_{k-1}(v_i)|} \le p |\Pi_{k-1}(v_i)|.
$$ 

On the one hand, there are at most $n$ neutral vertices. 
Therefore, conditional on $|\Pi_{k-1}(v_i)|$,  
$|\Pi_{k}(v_i)|$ is stochastically dominated by 
$
 \Bin \left( n , \; p \left|\Pi_{k-1}(v_i) \right|  \right)
$
and hence 
\begin{align}
\prob{   |\Pi_{k}(v_i)|  \ge  (4np)^{k} \mid \calF_1,\ldots, \calF_{k-1}, \calF_{k,i-1} }
&\le \prob{ \Bin \left(  n, p (4np)^{k-1 }\right) \ge (4np)^k } \nonumber \\
& \le  \exp \left( -  4^{k-1} (np)^k \right), \label{eq:neighbor_upper}
\end{align}
where the last inequality follows from the Binomial tail bound~\prettyref{eq:binomial_upper}.

On the other hand, in view of assumption \prettyref{eq:neighbor_p_assumption_2}, 
conditional on $\calF_1,\ldots, \calF_{k-1}, \calF_{k,i-1}$ 
there are at least 
$$
n - 1 - \sum_{i=1}^{d_u} 
\sum_{\ell=0}^{k-1} | \Pi_{\ell} (v_i) |
-\sum_{j=1}^{i-1}  |\Pi_{k}(v_j)| 
\ge n - 1 - 4np \sum_{\ell=0}^{k} (4np)^\ell  
= n-\frac{ (4 np)^{k+2} -1 }{4np-1} = n-o(n)
$$ 
neutral vertices to be connected to some vertices in $\Pi_{k-1}(v_i)$,
and for each  $v_i$ such that $|\Pi_1(v_i)| \ge \tau$, 
$$
 p_k = 1- (1-p)^{|\Pi_{k-1}(v_i)|} \ge  \left( 1-o(1) \right) p  \tau \left( \frac{np}{2} \right)^{k-2}. 
$$
Therefore, conditional on $\calF_1,\ldots, \calF_{k-1}, \calF_{k,i-1}$,
$|\Pi_{k}(v_i)|$ is stochastically lower bounded by  
$$
 \Bin \left( n -o(n) , \left( 1-o(1) \right) p  \tau \left( \frac{np}{2} \right)^{k-2} \right)
$$
and hence for $2 \le k \le d$,
\begin{align}
& \prob{   |\Pi_{k}(v_i)|  \ge  \tau \left( \frac{np}{2} \right)^{k-1}   \mid  \calF_1,\ldots, \calF_{k-1}, \calF_{k,i-1} } \nonumber \\
& \le \prob{ \Bin \left( n -o(n), \left( 1-o(1) \right) p  \tau \left( \frac{np}{2} \right)^{k-2} \right) \le \tau \left( \frac{np}{2} \right)^{k-1} }  \nonumber \\
& \le \exp \left( - \Omega \left(  \tau \left( \frac{np}{2} \right)^{k-1} \right) \right). \label{eq:neighbor_lower}
\end{align}
Combining \prettyref{eq:neighbor_upper} and \prettyref{eq:neighbor_lower} yields that
\begin{align*}
\prob{ \calF_{k,i} \mid  \calF_1,\ldots, \calF_{k-1} } \ge \prob{ \calF_{k,i-1} \mid  \calF_1,\ldots, \calF_{k-1} }
\left( 1- 2\exp \left( - \Omega \left(  \tau \left( \frac{np}{2} \right)^{k-1} \right) \right)
 \right).
\end{align*}
Therefore,
\begin{align*}
\prob{ \calF_k \mid  \calF_1,\ldots, \calF_{k-1} } 
 \ge 1- 8np \exp \left( - \Omega \left( \tau \left( \frac{np}{2} \right)^{k-1}  \right) \right). 
\end{align*}
Finally, we note that 
\begin{align*}
&\prob{\calF_d \cap \calF_{d-1} \cap \cdots \cap \calF_2 \mid \calF_1} \\
& =\prob{ \calF_2 \mid \calF_1 } 
\prob{\calF_3 \mid \calF_1, \calF_2}  
 \cdots  \prob{\calF_d \mid \calF_1, \ldots, \calF_{d-1} } \\ 
& \ge
\prod_{k=2}^{d} \left( 1- 8np \exp \left( - \Omega \left( \tau \left( \frac{np}{2} \right)^{k-1}  \right) \right) \right) \\
& \ge 1 -  8np
\sum_{k=2}^{d} \exp \left( - \Omega \left( \tau \left( \frac{np}{2} \right)^{k-1}  \right) \right) \\
& \ge 1- 8 np \exp \left( -\Omega(\tau np )\right).
\end{align*}
Moreover, by the definition of $\calA_u$,
we have
$$
\calA_u^c = \left(\calF_d \cap \calF_{d-1} \cap \cdots \cap \calF_2 \right)^c \cap \calF_1 .
$$
Hence,
\begin{align*}
\prob{\calA_u^c} & =
\prob{\calF_1} \prob{\left(\calF_d \cap \calF_{d-1} \cap \cdots \cap \calF_2 \right)^c \mid \calF_1 } \\
& \le \prob{\left(\calF_d \cap \calF_{d-1} \cap \cdots \cap \calF_2 \right)^c \mid \calF_1 } \\
& \le 8 np \exp \left( -\Omega(\tau np )\right),
\end{align*}
completing the proof. 
\end{proof}

The following lemma shows that with high probability, for all possible root vertex $u$,
it has at most one children $v$ with $|\Pi_1(v)| \le \tau$ for $\tau=o(np)$.
Let $A$ denote the adjacency matrix of $G.$
For three distinct vertices $u, v, w$, define
$$
\calB_{u,v,w} = \{ A_{u,v} =1, A_{u, w} =1\} \cap \{ |\Pi_1(v)| \le \tau \} \cap \{  \Pi_1(w) \le \tau \}. 
$$
and 
$
\calB= \cup_{u, v, w} \calB_{u,v,w}. 
$
\begin{lemma} \label{lmm:low_degree_children}
Assume 
\begin{align}
np \ge \log n, \text{ and } 
np =o(n^{1/2}), \text{ and } \tau=o(np).
\label{eq:low_degree_childen_assump}
\end{align}
Then 
$$
\prob{\calB \cap \calE } \le n^{-1+o(1)}.
$$
\end{lemma}
\begin{proof} 
By the union bound,
$$
\prob{ \calB \cap \calE}
 \le \sum_{u,v,w} \prob{ \calB_{u,v,w} \cap \calE}
$$
it reduces to bounding $\prob{ \calB_{u,v,w} \cap \calE}.$

Let $N_v$ and $N_w$ denote the number of 
neutral vertices in the branching process when $v$ and $w$ are popped from the head of the queue,
respectively. Then conditional on $N_v$ and $N_w$, 
$|\Pi_1(v)|$ and $|\Pi_1(w)|$ are independent and
$|\Pi_1(v)| \sim \Binom( N_v, p ) $  and $|\Pi_1(w)| \sim \Binom(N_w, p).$
On event $\calE,$ both $N_v$ and $N_w$ is at least 
$n-1-4np-(4np)^2=n-o(n)$ in view of the assumption 
$np=o(n^{1/2}).$ Therefore, 
\begin{align*}
&\prob{ \left\{ |\Pi_1(v)| \le \tau, \Pi_1(w) \le \tau\right\}\cap \calE \mid A_{u,v}=1, A_{u,w}=1 } \\
&=\sum_{i,j=1}^{n-o(n)} \prob{ 
\left\{ |\Pi_1(v)| \le \tau, \Pi_1(w) \le \tau, N_v=i, N_w=j  \right\} \cap \calE \mid A_{u,v}=1, A_{u,w}=1 } \\
& \le \sum_{i,j=n-o(n)}^{n} \prob{ |\Pi_1(v)| \le \tau, \Pi_1(w) \le \tau,  N_v=i, N_w=j \mid A_{u,v}=1, A_{u,w}=1 } \\
& = \sum_{i,j=n-o(n)}^{n} \prob{ N_v=i, N_w=j \mid A_{u,v}=1, A_{u,w}=1 } 
\prob{ |\Pi_1(v)| \le \tau, \Pi_1(w) \le \tau \mid N_v =i, N_w=j} \\
& = \sum_{i,j=n-o(n)}^{n} \prob{ N_v=i, N_w=j \mid A_{u,v}=1, A_{u,w}=1 }
\left( \prob{ \Binom\left( n-o(n), p \right) \le \tau } \right)^2\\
& \le \exp \left(- 2 (1- o(1) ) np \right),
\end{align*}
where the last inequality holds due to the Binomial tail bound~\prettyref{eq:binomial_lower}
and the assumption that $\tau=o(np).$
It follows that 
$$
\prob{ \calB_{u,v,w} \cap \calE} \le p^2 \exp \left( -2 (1- o(1) ) np \right) =o(1/n^3),
$$
where the last equality holds due to $np \ge \log n$.
\end{proof}

Now we are ready to prove our main proposition. Let $\calH_u$ denote the event that tree $T(u)$ satisfies 
\begin{enumerate}
\item $u$ has at most one children $v$ 
such that $|\Pi_1(v)| \le \tau$. 
\item For each children $v$ of $u$ with $|\Pi_1(v)| \ge \tau$, $|\Pi_{k}(v)| \ge \tau \left( \frac{np}{2} \right)^{k-1}$
for all $1 \le k \le d$. 
\end{enumerate}
Define $\calH=\cap_u \calH_u$. 
\begin{proposition}\label{prop:tree_process}
Suppose \prettyref{eq:neighbor_p_assumption_2} and
\prettyref{eq:low_degree_childen_assump} hold and $\tau \to \infty$. Then 
$$
\prob{\calH} \ge 1- 3n^{-1+o(1)}. 
$$
\end{proposition}
 \begin{proof}
 Note that 
 $$
 \left( \cap_u \calA_u \right) \cap \left(\calB^c \cup \calE^c \right) \cap \calE \subset \calH.  
 $$
 Hence, 
 $$
 \prob{\calH} \ge 1- \sum_u \prob{ \calA^c_u } - \prob{ \calB \cap \calE } - \prob{\calE^c }.
 $$
 In view of \prettyref{lmm:branching_process}, we have
 $$
 \prob{ \calA^c_u } \le n^{-\omega(1) }.
 $$
 By \prettyref{lmm:low_degree_children}, we have
 $$
 \prob{ \calB \cap \calE }  \le n^{-1+o(1)}.
 $$
By \prettyref{lmm:graph_properties}, we have
$\prob{\calE} \ge 1-1/n$. 
Then the conclusion readily follows. 
 \end{proof}

\section{Time Complexity of Algorithm~\ref{alg:graph_matching_ell_hop} } \label{app:time_complexity_1}

Recall that in Algorithm~\ref{alg:graph_matching_ell_hop}, we need to 
efficiently check whether there exist $m$ independent $\ell$-paths from a given vertex $i_2$ to a set of
$m$ seeded vertices $L \subset \Gamma^{G_2}_\ell(i_2)$ in $G_2$ and 
$m$ independent $\ell$-paths from a given vertex $i_1$ to the corresponding seed set $\pi_0(L) \subset \Gamma^{G_1}_\ell(i_1)$ in
$G_1$. Below we give the specific procedure to reduce this task to a maximum flow problem in a directed graph with source $i_1$ and sink $i_2$.

First, we explore the local neighborhood $N^{G_1}_\ell(i_1)$ of $i_1$ in $G_1$ up to radius $\ell$. We delete all the edges $(u,v)$ found if $(u,v)$ are at the same distance from $i_1$. Also, we direct all the edges $(u,v)$ from $u$ to $v$ if $u$ is closer to $i_1$ than $v$ by distance $1$. Afterwards, we get a local neighborhood of $i_1$, denoted by $\tilde{N}_{\ell}^{G_1}(i_1)$, with edges pointing away from $i_1$. Note that $\tilde{N}_{\ell}^{G_1}(i_1)$ is not exactly a tree because a vertex  may have multiple parents. 

Next, we repeat the above procedure for vertex $i_2$ in $G_2$ in exactly the same manner except that the edges are directed towards $i_2$. Let $\tilde{N}_{\ell}^{G_2}(i_2)$ denote the resulting local neighborhood of $i_2$. 

Finally, we take the graph union of $\tilde{N}_{\ell}^{G_1}(i_1)$ and 
$\tilde{N}_{\ell}^{G_2}(i_2)$, by treating seeded vertex $u \in \Gamma_{\ell}^{G_2}(i_2)$
with its correspondence $\pi_0(u) \in \Gamma_{\ell}^{G_1}(i_1) $ as the same vertex. 
All the other vertices, seeded or non-seeded, 
from the two different local neighborhoods are treated 
as distinct vertices. We denote the resulting graph union as $N_\ell (i_1, i_2)$. 

Recall that we aim to find independent (vertex-disjoint except for $i_1$) $\ell$-paths from $i_1$ to seeded vertices in 
$\Gamma_{\ell}^{G_1}(i_1)$. Thus, we need to enforce the constraint that every vertex other than $i_1$ can appear at most once. Similarly for $i_2$. To this end, we apply the following procedure. 
\begin{enumerate}
\item Split each vertex $u$ in $N_\ell(i_1,i_2)$ into to two vertices: $uin$ and $uout$;
\item Add an edge of capacity $1$ from $uin$ to $uout$;
\item Replace each other edge $(u,v)$ in $N_\ell(i_1,i_2)$ 
 with an edge  from $uout$ to $vin$ of capacity 1;
\item Find a max-flow from $i_1out$ to $i_2in$. 
\end{enumerate}

The idea behind this construction is as follows. Any flow path from the source vertex $i_1out$ to the sink vertex $i_2in$ must have capacity 1, since all edges have capacity 1. Since all capacities are integral, there exists an integral max-flow in which all flows are integers~\cite{ford1956maximal}.
No two flow paths can pass through the same intermediary vertex, because in passing through a vertex $u$ in the graph the flow path must cross the edge from $uin$ to $uout$, and the capacity here has been restricted to one.
Also, the flow path must pass exactly $2\ell$ distinct $uout$ vertices (including the source vertex $i_1out$, because all the edges are pointing away from $i_1out$ and towards $i_2in$.
Thus each flow path from $i_1out$ to $i_2in$ represents a vertex-disjoint $2\ell$-path from the source vertex $i_1$ to sink vertex $i_2$ in $N_\ell (i_1, i_2)$.  As a consequence, the max-flow from $i_1out$ to $i_2in$ corresponds to the maximum number, $m$, of independent $\ell$-paths from $i_2$ to a set of seeded vertices $L \subset \Gamma^{G_2}_\ell(i_2) $ in $G_2$, and 
of independent $\ell$-paths from $i_1$ to the corresponding seed set $\pi_0(L) \subset \Gamma^{G_1}_\ell(i_1)$ in
$G_1$. 

As for time complexity, we can find a  max-flow from $i_1out$ to $i_2in$ via Ford--Fulkerson algorithm~\cite{ford1956maximal} in $O( |E| f)$ time steps, where $|E|$ is the total number of edges of $N_\ell (i_1, i_2)$ after vertex splitting and edge replacement,
and $f$ is the max flow. Under the choice of $\ell$ given in \prettyref{eq:def_ell_1}, the total number of vertices and
edges in $N_\ell(i_1,i_2)$ are $O(n^{1/2-\epsilon})$. Hence, $|E| = O(n^{1/2-\epsilon})$. Moreover, the max flow $f$ is upper bounded by the number of seeded vertices in  $\Gamma_{\ell}^{G_1}(i_1)$ which is at most $O(n^{1/2-\epsilon} \alpha)$ with high probability.  Hence, in total it takes $O(n \alpha)$ time steps to compute the
max-flow from $i_1out$ to $i_2in$ via Ford--Fulkerson algorithm.

\section{Tail Bounds for Binomial Distributions}
\begin{theorem}[\cite{Okamoto1959,Mitzenmacher05}]\label{thm:binom_lower_tail}
Let $X \sim \Bin(n,p)$. It holds that
\begin{align}
\prob{X \le n t} &\le \exp \left( - n \left( \sqrt{p} - \sqrt{t} \right)^2\right), \quad \forall 0 \le t \le p 
\label{eq:binomial_lower}\\
\prob{X \ge n t} & \le \exp \left( - 2n  \left( \sqrt{t} - \sqrt{p} \right)^2\right), \quad \forall p \le t \le 1.
\label{eq:binomial_upper}\\
\prob{X \ge n t} &  \le 2^{- nt}, \quad \forall 6p \le t \ge 1.
\label{eq:binomial_upper_2}
\end{align}
\end{theorem}

\end{appendices}

\section*{Acknowledgment}
 J.~Xu would also like to thank Cris Moore, Jian Ding, Zongming Ma, and Yihong Wu 
for inspiring discussions
on graph matching and isomorphism. J.~Xu was supported by the NSF Grant CCF-1755960.

\bibliographystyle{alpha}
\bibliography{graphical_combined}

\end{document}